\documentclass{article}

\usepackage{amsthm}

\newtheorem{innercustomgeneric}{\customgenericname}
\providecommand{\customgenericname}{}
\newcommand{\newcustomtheorem}[2]{%
  \newenvironment{#1}[1]
  {%
   \renewcommand\customgenericname{#2}%
   \renewcommand\theinnercustomgeneric{##1}%
   \innercustomgeneric
  }
  {\endinnercustomgeneric}
}

\newcustomtheorem{customthm}{Theorem}
\newcustomtheorem{customlemma}{Lemma}

\usepackage[preprint]{neurips_2022}

\usepackage[inline]{enumitem}
\usepackage[utf8]{inputenc} 
\usepackage[T1]{fontenc}    
\usepackage{hyperref}       
\usepackage{url}            
\usepackage{booktabs}       
\usepackage{amsfonts}       
\usepackage{nicefrac}       
\usepackage{microtype}      
\usepackage{xcolor}         
\usepackage{array}
\usepackage{amsthm}
\usepackage{amsmath,amssymb,amsfonts}
\usepackage{algorithmic}
\usepackage{graphicx}
\usepackage{textcomp}
\usepackage{graphicx}
\usepackage{amsmath}
\usepackage{amssymb}
\usepackage{mathtools}
\usepackage{amsthm}
\usepackage{booktabs}
\usepackage{amsfonts}       
\usepackage{nicefrac}       
\usepackage{tabularx}
\usepackage{microtype}      
\usepackage{xcolor}         
\usepackage{amsmath}
\usepackage{amsmath}
\usepackage{float}
\usepackage{graphicx}
\usepackage{subfig}
\usepackage{amssymb}
\usepackage{bbm}
\usepackage{enumitem}
\newtheorem{theorem}{Theorem}[section]
\newtheorem{lemma}{Lemma}

\newtheorem{remark}{Remark}
\newtheorem{definition}{Definition}
\newtheorem{assumption}{Assumption}
\newtheorem{corollary}[theorem]{Corollary}

\usepackage{multirow}
\usepackage{caption} 
\newcommand{\vertiii}[1]{{\left\vert\kern-0.25ex\left\vert\kern-0.25ex\left\vert #1 
    \right\vert\kern-0.25ex\right\vert\kern-0.25ex\right\vert}}
    
\newcommand{\bE}{\mathbb{E}}
\newcommand{\cF}{\mathcal{F}}
\newcommand{\tr}{T}

\title{Does Momentum Help in Stochastic Optimization? \\ A Sample Complexity Analysis.}

\author{%
  Swetha Ganesh $^{*}$ \\
  Department of Computer Science and Automation \\
  Indian Institute of Science, Bangalore\\
  \texttt{swethaganesh@iisc.ac.in}
  \And
  Rohan Deb \thanks{Equal contribution.} \\
  Department of Computer Science and Automation \\
  Indian Institute of Science, Bangalore\\
  \texttt{rohandeb@iisc.ac.in} \\
  \AND
  Gugan Thoppe \\
  Department of Computer Science and Automation \\
  Indian Institute of Science, Bangalore\\
  \texttt{gthoppe@iisc.ac.in}
  \And
  Amarjit Budhiraja\\
  Department of Statistics and Operations Research \\
  University of North Carolina at Chapel Hill\\
  \texttt{budhiraj@email.unc.edu} \\
}

\begin{document}

\maketitle

\begin{abstract}
Stochastic Heavy Ball (SHB) and Nesterov's Accelerated Stochastic Gradient (ASG) are popular momentum methods in stochastic optimization. While benefits of such acceleration ideas in deterministic settings are well understood, their advantages in stochastic optimization is still unclear. In fact, in some specific instances, it is known that momentum does not help in the sample complexity sense. Our work shows that a similar outcome actually holds for the whole of quadratic optimization. Specifically, we obtain a lower bound on the sample complexity of SHB and ASG for this family and show that the same bound can be achieved by the vanilla SGD. We note that there exist results claiming the superiority of momentum based methods in quadratic optimization, but these are based on  one-sided or flawed analyses.
\end{abstract}

\section{Introduction}
\label{sec_introduction}

In deterministic convex optimization (when one has access to exact gradients), Gradient Descent (GD) is a popular optimization algorithm \cite{cauchy}. However, in practice, exact gradients are not available and one has to rely on noisy observations. This brings forth the idea of Stochastic Gradient Descent (SGD). For GD, two classic momentum ideas, Heavy Ball (HB) \cite{polyak_heavy_ball, polyakbook, qian} and Nesterov's Accelerated Gradient (NAG) \cite{nesterov, nesterovbook, Nesterov05}, are used to speed up its convergence. Naturally, these momentum-based methods and their variants have also gained significant interest in stochastic settings \cite{sutskever13, Nitanda2014, Chonghai}. Our work shows that the stochastic variants of HB and NAG, i.e., the Stochastic Heavy Ball (SHB) and Nesterov's Accelerated Stochastic Gradient (ASG), are not better than the vanilla SGD for the whole of quadratic optimization. Specifically, we show that the sample complexities\footnote{Sample complexity refers to the number of iterations required to reach an $\epsilon$-boundary of the solution.} of SHB, ASG, and SGD are all of the same order. 

More formally, in (deterministic) quadratic optimization\footnote{Throughout this work, we only consider algorithms with constant step-sizes. These are popular in practice and ensure faster convergence.} with condition number $\kappa,$ GD converges to an $\epsilon$-optimal solution in $\mathcal{O}(\kappa\log\frac{1}{\epsilon})$ iterations, while both HB and NAG only need $\mathcal{O}(\sqrt{\kappa}\log\frac{1}{\epsilon})$ steps. A pertinent question then is ``Does one get such advantages even in stochastic settings?" The current literature is, however, divided on whether SHB and ASG are better than SGD. 

Some recent results \cite{Loizou, MJ, Assran, zhu2} 
claim that these momentum methods are better than SGD in quadratic or least-squares settings. However, \cite{Loizou} needs a strong assumption on noise, which \cite[Section 6]{rahul} claims is information-theoretically impossible even in the simple least squares regression problem. The other results either are based on a one-sided analysis \cite{zhu2} (only considers bias, while ignoring variance) or have a major flaw \cite{MJ, Assran}; see Appendix A.1---A.3 for details.

On the other hand, there are also a few recent negative results on these momentum methods.
For general convex optimization, \cite{ucla} 
shows that SHB and ASG is equivalent to SGD with a rescaled step-size. 
However, this result requires that the stepsize be sufficiently small and the momentum parameter to be away from $1.$ In \cite{rahul, liu}, for one specific instance of the least squares regression with vanishing noise, it is shown that the performance of SHB and ASG cannot be better than that of SGD.  
Finally, \cite{nqm} considers SHB for quadratic objectives in the noisy setting as our work and provides upper bounds on the rate at which the objective function decreases. They also argue that rescaled SGD performs as well as SHB and demonstrate it empirically but fall short of rigorously coming up with a lower bound that supports their claim.

The current literature can thus be summarized as follows.

\textbf{Research Gap}: Existing works on SHB and ASG fall into two groups: i.) positive -  where the results claim advantages of these methods over SGD and ii.) negative - where the results claim the opposite. Results in the positive group either have a one-sided \cite{zhu2} or a flawed analysis \cite{MJ, Assran}, while the ones in the negative apply to special cases (e.g., \cite{ucla} requires sufficiently small stepsizes and momentum parameters away from $1,$ while \cite{rahul, liu} only apply to one instance of a least squares problem with vanishing noise). 

\textbf{Key Contribution:} Our work belongs to the negative group and extends the claims in \cite{rahul, liu} to more general settings.
Specifically, for all quadratic optimization problems with persistent noise (noise variance is bounded away from zero) and any small $\epsilon > 0,$ we show that number of iterations needed by SHB and ASG to find an $\epsilon$-optimal solution are of the same order as that of SGD. More technically, we obtain a lower bound on sample complexities of SHB and ASG (Theorem \ref{Lower_bound}) and show that these are of the same order as the corresponding upper bound for SGD (Theorem \ref{Theorem_OTS_mom}). We emphasize that our result applies for all positive stepsizes and momentum parameters in the range $[0, 1]$ unlike in \cite{ucla}. Moreover, our proof techniques are also significantly different from those used in existing lower bounds such as \cite{rahul,liu}. This is because, under persistent noise, the expected error would contain an additional term that cannot be accounted for from their analyses (See Remark \ref{R3}).

\section{Main Results}
\label{sec:main}
We state our main results here that provide lower and upper bounds on the sample complexities of SHB and ASG. We use these bounds along with those of SGD to show that all these methods need a similar effort to find an $\epsilon$-optimal solution. We emphasize that our results apply to the whole family of quadratic optimization. 

To provide a unified analysis, we look at a generic update rule that includes as special cases SHB, ASG, and SGD. In particular, we consider the linear stochastic approximation with momentum (LSA-M) iterate given by
\begin{align}
    \label{sgd-m}
    x_{n+1} = {} & x_{n} + \alpha(b - Ax_{n} + M_{n+1}) \nonumber\\
    & +\eta(I-\alpha\beta A)(x_{n} - x_{n-1}).\\
    \label{sgd-m-2}
    = {} & x_{n} + \alpha(b - A(x_{n}+\eta \beta (x_{n} -x_{n-1})) + M_{n+1})\nonumber\\ 
    & + \eta (x_{n} - x_{n-1}),
\end{align}
where $b\in\mathbb{R}^d, A \in\mathbb{R}^{d\times d},$ and $M_{n+1} \in \mathbb{R}^d$ is noise. Note that we do not presume $A$ is symmetric and this is what makes the above update a stochastic approximation. Also, when $A$ is symmetric, LSA-M subsumes as special cases SGD (let $\eta = 0$ in \eqref{sgd-m}), SHB (let $\beta = 0$ in \eqref{sgd-m}), and ASG (let $\beta=1$ in \eqref{sgd-m-2}). 


%
With regards to $A$ and the noise sequence $(M_{n}),$ we make the following assumptions.
\begin{assumption}
\label{A1}
(Driving matrix feature)
$A$ is diagonalizable and all its eigen values \(\{\lambda_{i}(A)\}_{i=1}^{d}\) are real and positive.
\end{assumption}

Assumption~\ref{A1} holds trivially when $A$ is a symmetric positive definite matrix. Consequently, our results apply to all members of the quadratic optimization family.

Also, under the above assumption, one would expect the iterates in \eqref{sgd-m} to go to a neighborhood of $x^* := A^{-1} b.$

We next state two assumptions for $(M_n):$ the first is used in Theorem~\ref{Lower_bound}, while the other is used in Theorem~\ref{Theorem_OTS_mom} and Corollary~\ref{Upper_bound_cor}.

\edef\oldassumption{\the\numexpr\value{assumption}+1}

\setcounter{assumption}{0}
\renewcommand{\theassumption}{\oldassumption.\alph{assumption}}

\begin{assumption}
\label{A2}
(Noise attributes for Theorem~\ref{Lower_bound}) $(M_{n})$ is a martingale difference sequence w.r.t. the filtration $(\mathcal{F}_{n})$, where $\mathcal{F}_{n} = \sigma(x_{m},M_{m};m\leq n)$. Further, $\exists K > 0$ such that $\mathbb{E}[M_{n+1} M_{n+1}^T|\mathcal{F}_{n}]\succeq K I_d$ a.s. for all $n \geq 0.$
\end{assumption}

The notation $A \succeq B$ for $A,B \in \mathbb{R}^d$ is used above to imply that $A-B$ is positive semi-definite. Also, the symbol $I_d$ refers to the identity matrix. 

\begin{assumption}
\label{A3}
(Noise attributes for Theorem~\ref{Theorem_OTS_mom})
$(M_n)$ is a martingale difference sequence w.r.t the filtration $(\mathcal{F}_{n})$, where $\mathcal{F}_{n} = \sigma(x_{m},M_{m};m\leq n)$. Further, $\exists K \geq 0$ such that  $\mathbb{E}[\|M_{n+1}\|^{2}|\mathcal{F}_{n}] \leq K(1+\|x_{n}-x^{*}\|^2)$ a.s. for all $n \geq 0.$
\end{assumption}

\let\theassumption\origtheassumption

Assumptions~\ref{A2} and \ref{A3} are standard \cite{mandt, jas, cheng, Borkar_Book}. The first of these holds if and only if all the eigenvalues of $\bE[M_{n + 1} M_{n + 1}^\tr |\cF_n]$ are bounded from below by $K,$ i.e., noise is persistent throughout (or non-vanishing) in all directions. On the other hand, Assumption~\ref{A3} requires that the trace of $\bE[M_{n + 1} M_{n + 1}^\tr |\cF_n]$ be bounded from above. This bound can scale with $\|x_n - x^*\|$ and need not vanish near $x^*.$

Next, we define sample complexity which quantifies the effort required by LSA-M to obtain an $\epsilon$-close solution to $x^*.$

\begin{definition}
(Sample Complexity). The sample complexity of \eqref{sgd-m} is the minimum number of iterations $n_0$ such that the expected error $\mathbb{E}[\|x_n-x^*\|^2] \leq \epsilon$  $ \forall n \geq n_0$. 
\end{definition}

\begin{table*}[t]
\begin{center}
\begin{tabular}{ |c|c|c|c| }
\hline
 Method &$\beta$ & $\eta$ & $\alpha$ \\ \hline
\multirow{3}{*}{SGD}
 & & & \\
 & - & 0 & $\min\Big(\frac{\lambda_{min}(A)}{\frac{3}{4}\lambda_{min}(A)^{2} + C^{2}K}, \frac{\epsilon\lambda_{min}(A)}{4C^2K}, \frac{2}{\lambda_{max}(A)+\lambda_{min}(A)} \Big) $ \\
 & & & \\

 \hline
\multirow{1}{*}{SHB}
 & 0 & $ \left(1-\frac{\sqrt{\alpha\lambda_{min}(A)}}{2}\right)^{2} $ &$ \min\Big((\frac{\lambda_{\min}(A)^{\frac{3}{2}}}{\frac{3}{8}\lambda_{\min}(A)^{2} + 25C^2K})^{2}, (\frac{\epsilon(\lambda_{\min}(A))^{3/2}}{200 C^2 K})^2,$ \\ 
 & &  &  $(\frac{2}{\sqrt{\lambda_{min}(A)}+\sqrt{\lambda_{max}(A)}})^2\Big) $ \\
 & & & \\
  
 \hline
 \multirow{1}{*}{ASG}
 & 1 & $\frac{ \left(1-\frac{\sqrt{\alpha\lambda_{min}(A)}}{2}\right)^{2}}{(1-\alpha\lambda_{min}(A))}$ & $\min\Big((\frac{\lambda_{\min}(A)^{\frac{3}{2}}}{\frac{3}{8}\lambda_{\min}(A)^{2} + 25C^2K})^{2}, (\frac{\epsilon(\lambda_{\min}(A))^{3/2}}{200 C^2 K})^2,\frac{1}{\lambda_{max}(A)}\Big)$\\
 & &  &   \\
  
 \hline
\end{tabular}
\\
\caption{Parameter choices for Theorem
\ref{Theorem_OTS_mom}. Here $C=1$ when the matrix $A$ is symmetric and $C = \frac{\sqrt{d}}{\sigma_{\min}(S)\sigma_{\min}(S^{-1})}$ when $A$ is not symmetric, where $\sigma_{min}(\cdot)$ denotes the smallest singular value and $S$ is the matrix that diagonalizes $A$, i.e., $S^{-1}AS = D$, a diagonal matrix. When $A$ is symmetric, indeed the three parameter choices correspond to SGD, SHB and ASG. However, with a slight abuse of notation, we stick to the same naming convention even when the driving matrix $A$ is not symmetric.}
\label{table}
\end{center}
\end{table*}

To enable easy comparison between different algorithms, we shall look at the order of their sample complexities. Towards that, we shall use the notation $n_0 \in \Theta(t)$ to imply that there exist constants $c_1$ and $c_2$ (independent of $t$) such that $c_1 t \leq n_0 \leq c_2 t $. The notation $\tilde{\Theta}(t)$ has a similar meaning but hides the dependence on the logarithmic terms as well.

\begin{theorem}[Lower bound on sample complexity]
    \label{Lower_bound}
    Consider the LSA-M iterate in \eqref{sgd-m} and suppose Assumptions \ref{A1} and \ref{A2} hold. Let $\epsilon > 0$ be small enough and $\lambda_{min}(A) = \min_{i}\lambda_{i}(A)$. Then, for any choice of $\alpha > 0,$ $\beta \in [0,1],$ and $\eta \in [0,1]$, the expected error $\mathbb{E}[\|x_{n_{0}} - x^{*}\|^2]\geq\epsilon$ for $n_0 = \frac{K}{64\epsilon \lambda_{min}(A)^2}\log(\frac{\|x_0-x^*\|^2}{\epsilon}) \in \tilde{\Theta}\left(\frac{K}{\epsilon\lambda_{min}(A)^2}\right).$
\end{theorem}
\begin{remark}
    \label{R1}
    As stated below \eqref{sgd-m-2}, LSA-M includes SHB and Nesterov's ASG method as special cases and, hence, the above result directly applies to them. In fact, this is the first lower bound on SHB and ASG's sample complexities in quadratic optimization with persistent noise.
\end{remark}
\begin{remark}
\label{R2}
    The value of $n_0$ above is obtained by optimizing over all possible values of $\alpha, \beta, \eta$ and hence does not depend on their specific choices. This implies that there is no ``smart'' combination of $\alpha,\beta$ and $\eta$ that will give a better sample complexity (for $\epsilon$ small enough).
\end{remark}
\begin{remark}
\label{R3}
    The lower bounds in \cite{rahul} and \cite{liu} are obtained by viewing the expected error in SHB and ASG iterates for least squares as update rules of the form $z_{n + 1} = P z_n$ for some matrix $P$ \cite[Appendix~A, p~16]{rahul} and \cite[Appendix C, p~12]{liu}). In particular, they obtain bounds on the eigenvalues of $P$ to get the desired claim. In contrast, the error relations for SHB and ASG methods in our setup (quadratic optimization with persistent noise) have the form $z_{n + 1} = P z_n + \alpha W_n$ for some matrix $P$ and vector $W_n$ (cf. \ref{recursion}). This forces us to develop a new proof technique that jointly looks at both these terms and show that at least one of them remains larger than $\epsilon$ for the choice of $n_0$ given in Theorem~\ref{Lower_bound}. 
\end{remark}

%

We next state our upper bound on the sample complexity in Theorem~\ref{Theorem_OTS_mom} and Corollary~\ref{Upper_bound_cor}. Similar bounds already exist in literature when $A$ is assumed to be symmetric and the noise is assumed to be iid with bounded variance (\cite{zhu2,nqm}). Here, we show that a similar upper bound holds under more general settings: i.) $A$ is not symmetric but satisfies Assumption~\ref{A1}, and ii.) the noise is a martingale difference sequence satisfying Assumption~\ref{A3}.

\begin{theorem}
\label{Theorem_OTS_mom} Consider the LSA-M iterate and let Assumptions~\ref{A1} and \ref{A3} hold. 
Then, $\forall \epsilon > 0$, there exists a choice of
$\alpha$, $\beta$ and $\eta$ in LSA-M (see Table~\ref{table} for exact values) such that the expected error $\mathbb{E}[\| x_{n}-x^*\|^2] \leq \epsilon$, $\forall n > n_{0}$, where,

\begin{enumerate}
    \item[(i)] $n_{0} \in \Tilde{\Theta}(\frac{1}{{\alpha\lambda_{min}(A)}}),$ when $\eta = 0$ 
    \item[(ii)] $n_{0} \in \Tilde{\Theta}(\frac{1}{\sqrt{\alpha\lambda_{min}(A)}}),$ when $\eta > 0$.
\end{enumerate}
\end{theorem}
From Table \ref{table}, we see that for each case $\alpha$ is a minimum of three terms. The first term arises due to the unbounded noise (Assumption \ref{A3}), the second due to the target neighborhood $\epsilon$  and the third from the optimal choice of step-size in the deterministic case. Since the bound provided in Theorem \ref{Theorem_OTS_mom} is in terms of $\alpha$, the minimum of the three terms dictate the sample complexity. If $\epsilon$ is small enough such that the second term is the minimum, we obtain the following bound:
\begin{corollary}[Upper bound on sample complexity]
    \label{Upper_bound_cor}
    Consider the LSA-M iterate and suppose Assumptions~\ref{A1} and \ref{A3} hold. Then for choice of parameters in Table~\ref{table}$, \forall  \epsilon>0$ small enough, $\exists n_0 \in \Tilde{\Theta}\left(\frac{K}{\epsilon\lambda_{min}(A)^2}\right)$ such that $\mathbb{E}[\| x_{n}-x^*\|^2] \leq \epsilon$, $\forall n \geq n_{0}$.
\end{corollary}
\begin{remark}
    \label{R4}
    From Corollary \ref{Upper_bound_cor}, we see that for small enough $\epsilon>0$ the upper bound on the sample complexity of SGD, SHB and ASG match their lower bound in Theorem \ref{Lower_bound}. In particular, since an upper bound on the sample complexity of SGD matches a lower bound on SHB and ASG, these methods cannot outperform SGD from a sample complexity perspective.
\end{remark}
\begin{remark}
    \label{R5}
    When the noise is assumed to be bounded, the first term in the choice of $\alpha$ in Table~\ref{table} vanishes for all three parameter choices. Under such an assumption, if $\epsilon$ is large enough or $K$ is small enough such that the third term in the choice of $\alpha$ is the minimum, then the sample complexity of both SHB and ASG is better than SGD. We emphasize that such improvements are lost when the noise variance is large or the neighbourhood under consideration is small.
\end{remark}

\section{Proof Outline}
\label{sec:proof}
In this section we provide the proof outline for Theorem~\ref{Lower_bound} and Theorem~\ref{Theorem_OTS_mom}. Section~\ref{proof1} provides an outline of all the key steps to prove the lower bound on the sample complexity. The proof of all the intermediate lemmas have been pushed to Appendix B. In section~\ref{proof2} we provide a sketch of the proof of Theorem~\ref{Theorem_OTS_mom}. The detailed proof can be found in Appendix C.
\subsection{Proof Outline for Theorem~\ref{Lower_bound}}
\label{proof1}
For simplicity, let $d = 1$ and $A\equiv\lambda \in \mathbb{R}$. The general multivariate case is handled subsequently.
We begin by defining the transformed iterates $\Tilde{x}_{n} = x_{n} - x^*$ and rewrite \eqref{sgd-m} as:
\begin{equation}
\label{recursion}
    \Tilde{X}_{n} = P \Tilde{X}_{n-1} + \alpha W_{n},
\end{equation}
where,
\(\Tilde{X}_{n} \triangleq
\begin{bmatrix}
    \Tilde{x}_{n} \\
    \Tilde{x}_{n-1}      
\end{bmatrix},
W_{n} \triangleq
\begin{bmatrix}
    M_{n}\\
    0
\end{bmatrix}\)
and
\(
P \triangleq
\begin{bmatrix}
    1 - \alpha \lambda + \eta(1-\alpha\beta\lambda) & -\eta(1-\alpha\beta\lambda)   \\
    1 & 0      
\end{bmatrix}. 
\)
For ease of exposition, we first consider the case when $\beta = 0$, corresponding to SHB and consider the general case for $\beta \in [0,1]$ later. The objective here is to find an $n$ such that the error $\mathbb{E}[\|\tilde{X}_{n}\|^2]$ is lower bounded by $\epsilon$. Towards this, we decompose the error expression as follows:
\begin{align}
\label{general-lower-bound}
    \mathbb{E}[\| \tilde{X}_n\|^2] &= \mathbb{E}[\| P^n \Tilde{X}_{0} + \alpha \sum_{j=0}^{n-1}P^{n-1-j}W_{j+1} \|^2] \nonumber \\
    &= \| P^n  \Tilde{X}_{0} \|^2 + \mathbb{E}[\alpha^2   \sum_{j=0}^{n-1} \|P^{n-1-j}W_{j+1} \|^2 ] \nonumber \\
    &\geq \| P^n  \Tilde{X}_{0} \|^2 + \alpha^2  K \sum_{j=0}^{n-1} \|P^{j} e_1 \|^2, 
\end{align}
where, $e_1=\begin{pmatrix}
1\\
0
\end{pmatrix}$. Here, the second equality follows since each $W_{j+1}$ is a martingale difference sequence and the inequality follows from Assumption \ref{A3}. The first term here corresponds to the \emph{bias} and the second term corresponds to the \emph{variance}.



Let $n=\frac{K}{64\epsilon \lambda^2}\log(\frac{\|\tilde{X_0}\|^2}{\epsilon})$. We will show that for this choice of $n$, the expected error will always be greater than or equal to $\epsilon$ for all $\alpha >0$ and $\eta \in [0,1]$. From \eqref{general-lower-bound}:
\begin{gather*}
    \mathbb{E}\|\tilde{X}_n \|^2
    \geq \|P^n \tilde{X_0} \|^2 + \alpha^2 K \sum_{j=0}^{n-1} \Big\|P^{j}e_{1}\Big\|^2.
\end{gather*}
If $\alpha$ is such that $\|P^n \tilde{X_0} \|^2 \geq \epsilon$, then the claim immediately follows for this choice of $\alpha$ and $\eta$. We now consider the case where $\alpha$ and $\eta$ are such that $\|P^n \tilde{X_0} \|^2 < \epsilon$. Here we will show that for this choice of $\alpha$ and $\eta$, the \emph{variance} will necessarily be greater than $\epsilon$. Let $\mu_{+}$ and $\mu_{-}$ be the eigen-values of $P$.


\begin{lemma}
\label{lemma1}
    For all $\alpha>0$, $\eta \in [0,1]$, if $\|P^n \tilde{X_0} \|^2 < \epsilon$ for small enough $\epsilon>0$ then the variance term $\alpha^2 K \sum_{j=0}^{n-1} \Big\|P^{j}e_{1}\Big\|^2$ is bounded by $$\alpha^2 K \sum_{j=0}^{n-1} \Big\|P^{j}e_{1}\Big\|^2 \geq \frac{\alpha^2 K}{2(1-\mu_+^2)(1-\mu_-^2)(1-\eta)}.$$
\end{lemma}
Note that, from the above lemma, the \emph{variance} bound improves (i.e., decreases) when $\eta$ decreases. Using this, we show that the \emph{variance} is ultimately bounded below by a linearly decreasing function of $\rho(P)$ in the following lemma.
\begin{lemma}
\label{lemma2}
    For all $\alpha >0$, $\eta \in [0,1]$,  if $\|P^n \tilde{X_0} \|^2 < \epsilon$ for small enough $\epsilon>0$, then $\alpha^2 K \sum_{j=0}^{n-1} \Big\|P^{j}e_{1}\Big\|^2 \geq \frac{K}{16 \lambda^2}(1-\rho(P)).$
\end{lemma}
Recall that $n=\frac{K}{64\epsilon \lambda^2}\log(\frac{\|\tilde{X_0}\|^2}{\epsilon})$ and $\|P^n\tilde{X}_{0} \|^2 < \epsilon$. Using this, we get the following lower bound on $1-\rho(P)$:
\begin{lemma}
\label{lemma3}
     If $n=\frac{K}{64\epsilon \lambda^2}\log(\frac{\|\tilde{X_0}\|^2}{\epsilon})$ and $\|P^n\tilde{X}_{0} \|^2 < \epsilon$, then $1-\rho(P) \geq \frac{16\epsilon \lambda^2}{K}$.
\end{lemma}


\noindent This completes the proof for the univariate case when $\beta=0$.

We now consider the case where $\beta\in[0,1]$. We notice that in the univariate case, LSA-M with parameters $\beta$ and $\eta$ is equivalent to LSA-M with parameters $\beta=0$ and $\eta'=\eta (1-\alpha \lambda \beta)$.  If $\alpha \lambda > 1$, regardless of choice of $\eta$, we have from \eqref{general-lower-bound}:
\begin{align*}
    \mathbb{E}\|\tilde{X}_n \|^2 \geq \|P^ne_1 \|^2 + \alpha^2 K \sum_{j=0}^{n-1}\|P^{j}e_{1}\|^2  \geq  \alpha^2 K >  \frac{K}{\lambda^2}   
\end{align*}
Thus, when $\epsilon < \frac{K}{\lambda ^2}$, the expected error will always be larger than $\epsilon$. When $\alpha \lambda < 1$, then $\eta'=\eta (1-\alpha \lambda \beta) \in [0,1]$ for all $\eta, \beta \in [0,1]$. Since this is equivalent to LSA-M with $\beta=0$ and $\eta' \in [0,1]$, the result follows from the lower bound on univariate LSA-M with $\beta=0$. This completes the proof for the univariate case for all $\alpha>0$, $\eta, \beta \in [0,1]$.

Finally, we consider the multivariate case, i.e., when $d>1$. 
For the multivariate case, we study the transformed iterates $\tilde{Y_n} = Z \tilde{X_n}$, where $Z:=E_{d\times d}
\begin{pmatrix}
    S & 0 \\
    0 & S
\end{pmatrix}$. Since $\Tilde{X}_{n} = P \Tilde{X}_{n-1} + \alpha W_{n},$ we get
\begin{align*}
    \tilde{Y}_{n} &=Z\Tilde{X}_{n} = Z P \Tilde{X}_{n-1} + \alpha Z W_{n}
    =  Z P \Tilde{X}_{n-1} + \alpha Z W_{n} \\
    &= Z P Z^{-1}\Tilde{Y}_{n-1}+ \alpha Z W_{n}
    = B \Tilde{Y}_{n-1}+ \alpha Z W_{n},
\end{align*}
where $$B=\begin{pmatrix}
        B_1 & 0 &\ldots & 0\\
        0 & B_2 & \ddots & \vdots\\
        \vdots & \ddots & \ddots & 0 \\
        0 & \ldots & 0 & B_d
    \end{pmatrix}, \mbox{ and } B_{i} = 
\begin{pmatrix}
    1 + \eta - \alpha \lambda_{i} & -\eta\\
    1 & 0
\end{pmatrix}.
$$

Let $\tilde{Y}_n = \begin{pmatrix}
    \tilde{Y}_n^{(1)} \\
    \tilde{Y}_n^{(2)} \\
    \vdots \\
    \tilde{Y}_n^{(d)} \\
\end{pmatrix}$ and 
$\hat{M}_{n} = S M_{n} =\begin{pmatrix}
    \hat{M}_n^{(1)} \\
    \hat{M}_n^{(2)} \\
    \vdots \\
    \hat{M}_n^{(d)} \\
\end{pmatrix}$,
where each $\tilde{Y}_n^{(i)} \in \mathbb{R}^2$ and $\hat{M}_n^{(i)} \in \mathbb{R}$. We will show that the update rule for each $\tilde{Y}_n^{(i)}$ is the same as that of the univariate case. Notice that $$\tilde{Y}_{n}^{(i)} = B_i \Tilde{Y}_{n-1}^{(i)}+ \alpha \begin{pmatrix}
    \hat{M}^{(i)}_n \\
    0
\end{pmatrix},$$
where $\mathbb{E}[\hat{M}^{(i)}_n| \mathcal{F}_{n-1}]=0$ and $\mathbb{E}[(\hat{M}^{(i)}_n)^2| \mathcal{F}_{n-1}] \geq \sigma_{\min}^2(S) K$. This is because $\mathbb{E}[\hat{M}_n| \mathcal{F}_{n-1}]=\mathbb{E}[S M_n| \mathcal{F}_{n-1}]=0$ and
\begin{align*}
    \mathbb{E}[(\hat{M}^{(i)}_n)^2| \mathcal{F}_{n-1}] &\geq \min_{\| v\| =1} \mathbb{E}[v^T\hat{M}_n \hat{M}^{T}_n v| \mathcal{F}_{n-1}] \\
    &= \min_{\| v\| =1} \mathbb{E}[v^T S M_n M^{T}_n S^T v| \mathcal{F}_{n-1}] \\
    &\geq \min_{\| v\| =1} \sigma_{\min}^2(S) \mathbb{E}[v^T M_n M^{T}_n v| \mathcal{F}_{n-1}] \\
    &\geq \sigma_{\min}^2(S) K.
\end{align*}

Let $\lambda_1 = \min_{i} \lambda_i$. Then, from results in the univariate case, $ \exists \epsilon' > 0$ such that for all $\alpha>0$ and $\beta, \eta \in [0,1]$, $\mathbb{E}[\|\tilde{Y}_{n}^{(1)}\|^2] \geq \epsilon'$ for some $n \in \Theta (\frac{\sigma_{\min}^2(S) K}{\epsilon' \lambda_{\min}^2})$. Letting $\epsilon = \frac{\epsilon'}{\sigma_{\max}^2(S)}$, it follows that $\mathbb{E}[\|\tilde{Y}_{n}^{(1)}\|^2] \geq \sigma_{\max}^2(S) \epsilon $ for some $n \in \Theta (\frac{\sigma_{\min}^2(S)  K}{\sigma_{\max}^2(S) \epsilon \lambda_{\min}^2})=\Theta (\frac{K}{\epsilon \lambda_{\min}^2})$ We ignore the term $\frac{\sigma_{\min}^2(S)}{\sigma_{\max}^2(S)}$ as it is a constant, and it is in fact equal to $1$ when $A$ is symmetric (diagonalising matrices of a symmetric matrix are orthogonal).
Finally,
\begin{align*}
   \mathbb{E}[\|\tilde{X}_{n}\|^2] &\geq \sigma^2_{\min}(Z) \mathbb{E}[\|\tilde{Y}_{n}\|^2] \geq  \sigma^2_{\min}(S^{-1}) \mathbb{E}[\|\tilde{Y}_{n}^{(1)}\|^2]
   \\ &= \frac{1}{\sigma^2_{\max}(S)} \mathbb{E}[\|\tilde{Y}_{n}^{(1)}\|^2] \geq \epsilon, 
\end{align*}
for some $n \in \Theta (\frac{K}{\epsilon \lambda_{\min}^2})$. 
\subsection{Proof Outline of Theorem~\ref{Theorem_OTS_mom}}
\label{proof2}
Here we provide a quick sketch of the proof when $\eta = 0$ (corresponding to SGD) and $\beta = 0$ (corresponding to SHB). A similar analysis follows for $\beta = 1$ and the detailed proof of all the three cases can be found in Appendix C.

The LSA-M iterate in \eqref{sgd-m} with $\eta = 0$  corresponds to:
\begin{equation*}
    \begin{split}
        \Tilde{x}_{n} &=
        (I - \alpha A)\Tilde{x}_{n-1} + \alpha M_{n}\\
        &= (I - \alpha A)^{n}\Tilde{x}_{0} + \alpha\sum_{i=0}^{n-1}[(I-\alpha A)^{n-1-i} M_{i+1}]
    \end{split}
\end{equation*}
Using Assumption-\ref{A3} the expected norm of the iterates can be bounded as follows:
\begin{equation*}
    \begin{split}
        \mathbb{E}[\|\Tilde{x_n}\|^2] &\leq \|(I-\alpha A)^{n}\|^2\Lambda \\ &+ \alpha^2 K\sum_{i=0}^{n-1}\|(I-\alpha A)^{(n-1-i)}\|^2 (1 + \mathbb{E}[\|\Tilde{x}_{i}\|^2])
    \end{split}
\end{equation*}
We next use the following lemma to bound $\|(I-\alpha A)^i\|$.
\begin{lemma}
\label{norm_upper_bound}
Let, $M \in \mathbb{R}^{d\times d}$ be a matrix and $\lambda_{i}(M)$ denote the $i^{th}$ eigen-value of $M$. Then, $\forall \delta > 0$
\[\|M^{n}\|\leq C_{\delta}(\rho(M) + \delta)^n\]
where $\rho(M) = \max_{i}|\lambda_{i}(M)|$ is the spectral radius of $M$ and $C_{\delta}$ is a constant that depends on $\delta$. Furthermore, if the eigen-values of $M$ are distinct, then
\[\|M^{n}\|\leq C(\rho(M))^n\]
\end{lemma}
\begin{proof}
See Appendix-D.
\end{proof}
Using \textbf{Assumption \ref{A1}} and Lemma \ref{norm_upper_bound}, we have
\begin{equation*}
    \begin{split}
        \mathbb{E}[\|\Tilde{x}_{n}\|^2] &\leq C^2(1-\alpha\lambda_{\min}(A))^{2n}\Lambda \\ &+ C^2\alpha^2K\sum_{i=0}^{n-1}(1-\alpha\lambda_{\min}(A))^{2(n-1-i)}(1 + \mathbb{E}[\|\Tilde{x}_{i}\|^2])
    \end{split}
\end{equation*}
\begin{equation}
    \mbox{where, \quad } C = \frac{\sqrt{d}}{\sigma_{\min}(S)\sigma_{\min}(S^{-1})},
\end{equation}
$S$ is the matrix in Jordan decomposition of $A$ and $\sigma_{\min}(S)$ is the smallest singular value of $S$. 
We define the sequence $\{U_{k}\}$ as below:
\begin{align*}
    U_{k} &= C^2(1-\alpha\lambda_{\min}(A))^{2k}\Lambda \\ &+ C^2\alpha^2K\sum_{i=0}^{k-1}(1-\alpha\lambda_{\min}(A))^{2(k-1-i)}(1 + U_{i})
\end{align*}
Observe that $\mathbb{E}[\|\Tilde{x}_{n}\|^2] \leq U_{n}$ and that the sequence $\{U_{k}\}$ satisfies
\begin{align*}
    U_{k+1} &= (1-\alpha\lambda_{\min}(A))^{2}U_{k} \\&+ C^{2}K\alpha^{2}(1 + U_{k});\quad
    U_{0} = C^{2}\Lambda
\end{align*}
Therefore, we have 
\[U_{k+1} = \left((1-\alpha\lambda_{\min}(A))^{2} + C^{2}K\alpha^{2}\right)U_{k} + \alpha^{2}C^{2}K\]
Since $\alpha \leq \frac{\lambda_{\min}(A)}{\frac{3}{4}\lambda_{\min}(A)^{2} + C^{2}K}$, one can show that
\begin{equation*}
    \begin{split}
        U_{n} & \leq \left(1 - \frac{\alpha\lambda_{\min}(A)}{2}\right)^{2n}U_{0} + \alpha^{2}C^{2}K \frac{2}{\alpha\lambda_{min}(A)}\\
    \end{split}
\end{equation*}
We assume that $\alpha \leq \frac{1}{\lambda_{\min}(A)}$ and therefore $(1-\frac{\alpha\lambda_{\min}(A)}{2})^{2} \leq e^{-\alpha\lambda_{\min}(A)}$.
\begin{equation*}
    \begin{split}
        U_{n}\leq e^{-n\alpha\lambda_{\min}(A)}C^{2}\Lambda + \frac{2\alpha C^{2}K}{\lambda_{min}(A)}
    \end{split}
\end{equation*}
Choose
\(\alpha \leq \frac{\epsilon\lambda_{min}(A)}{4C^2K}.\)
Then,
\[\frac{2\alpha C^{2}K}{\lambda_{min}(A)} \leq \frac{\epsilon}{2}
\Rightarrow 
\mathbb{E}[\|\Tilde{x}_{n}\|^2] \leq U_{n} \leq \frac{\epsilon}{2} + \frac{\epsilon}{2} = \epsilon,\]
when the sample complexity is:
\[n = \frac{1}{\alpha\lambda_{\min}(A)}\log\left(\frac{2C^2 \Lambda}{\epsilon}\right)\]

As in proof of Theorem~\ref{Lower_bound}, we begin by transforming the iterate given by \eqref{sgd-m} into the following iterate for $\beta = 0$:
\begin{gather*}
    \Tilde{X}_{n} = P \Tilde{X}_{n-1} + \alpha W_{n}
\end{gather*}
where, 
\(\Tilde{X}_{n} \triangleq
\begin{bmatrix}
    \Tilde{x}_{n} \\
    \Tilde{x}_{n-1}      
\end{bmatrix},
\)
\(
W_{n} \triangleq
\begin{bmatrix}
    M_{n}\\
    0
\end{bmatrix},
\)
and
\(
\label{P_def}
    P \triangleq
\begin{bmatrix}
    I - \alpha A + \eta I & -\eta I   \\
    I & 0      
\end{bmatrix}. 
\)

Using Assumption \ref{A2}, we can obtain the following bound on the mean squared error:
\[
\mathbb{E}[\|\Tilde{X}_{n}\|^2] \leq \|P^n\|^2 \|\Tilde{X}_{0}\|^2+ \alpha^2K\sum_{i=0}^{n-1}\|P^{n-1-i}\|^2(1 + \mathbb{E}[\|\Tilde{X}_{i}\|^{2}])
\]
Here, the first term corresponds to the bias and the second term correspond to the variance. A spectral analysis of $P$ tells us that the best choice of $\eta$ is given by $\eta = (1-\sqrt{\lambda_{\min}(A)\alpha})^2$
with \(\alpha \leq \big(\frac{2}{\sqrt{\lambda_{min}(A)} + \sqrt{\lambda_{max}(A)}}\big)^{2}\). However, to simplify the analysis we choose the parameter $\eta = \big(1-\frac{\sqrt{\lambda_{\min}(A)\alpha}}{2}\big)^2$, which maintains the same rate of decay of the bias term.
Using \textbf{Assumption \ref{A1}} and Lemma 9, we show that
\begin{gather*}
    \mathbb{E}[\|\Tilde{X}_{n}\|^2] \leq \hat{C}^2\left(1-\frac{\sqrt{\lambda_{\min}(A)\alpha}}{2}\right)^{2n} \Lambda
    + \alpha^2\hat{C}^2K\sum_{i=0}^{n-1}\left(1-\frac{\sqrt{\lambda_{\min}(A)\alpha}}{2}\right)^{2(n-1-i)}\!\!\!\!\!\!\!\!\!\!\!\!\!\!(1 + \mathbb{E}[\|\Tilde{X}_{i}\|^{2}]),
\end{gather*}
where $\Lambda = \|\tilde{x}_{0}\|^2$ and $\hat{C}$ depends on $\alpha$ and $\lambda$. The dependence of the second term on $\tilde{X}_{i}$ appears because we do not assume the noise variance to be uniformly bounded by a constant. This makes analyzing the above recursion challenging. Towards this, we define a sequence $\{V_{n}\}$,
\begin{gather*}
    V_{n} = \hat{C}^2\left(1-\frac{\sqrt{\lambda_{\min}(A)\alpha}}{2}\right)^{2n} \Lambda + \alpha^2\hat{C}^2K\sum_{i=0}^{n-1}\left(1-\frac{\sqrt{\lambda_{\min}(A)\alpha}}{2}\right)^{2(n-1-i)}(1 + V_{i}).
\end{gather*}
Choosing
\begin{align*} 
    \alpha \leq \min\Big((\frac{\lambda_{\min}(A)^{\frac{3}{2}}}{\frac{3}{8}\lambda_{\min}(A)^{2} + 25C^2K})^{2}, (\frac{\epsilon(\lambda_{\min}(A))^{3/2}}{200 C^2 K})^2,(\frac{2}{\sqrt{\lambda_{min}(A)}+\sqrt{\lambda_{max}(A)}})^2\Big)
\end{align*}
ensures that
\(\mathbb{E}[\|\Tilde{X}_{n}\|^2] \leq V_{n} \leq  e^{-n\frac{\sqrt{\lambda_{\min}(A)\alpha}}{2}}\hat{C}^{2}\Lambda + \frac{4\alpha^{2}\hat{C}^{2}K}{\sqrt{\alpha\lambda_{min}(A)}}.
\)
Using Lemma 10, we bound $\hat{C}$ in terms of $C$. Here, $C=1$ when the matrix $A$ is symmetric and $C = \frac{\sqrt{d}}{\sigma_{\min}(S)\sigma_{\min}(S^{-1})}$ when $A$ is not symmetric, where $\sigma_{min}(\cdot)$ denotes the smallest singular value and $S$ is the matrix that diagonalizes $A$, i.e., $S^{-1}AS = D$, a diagonal matrix. With this we obtain
\[V_{n} \leq  e^{-n\frac{\sqrt{\lambda_{\min}(A)\alpha}}{2}}\frac{25 C^{2}}{\lambda_{\min}(A)\alpha}\Lambda + \alpha^{2}\frac{100 C^{2}K}{\left(\lambda_{\min}(A)\alpha\right)^{3/2}}.\]
Now, for \(n \geq \frac{4}{\sqrt{\alpha\lambda_{\min}(A)}}\log\left(\frac{1}{\lambda_{\min}(A)\alpha}\right),\) it can show that
\(V_{n} \leq 25 C^{2}\Lambda e^{-\frac{n}{4}\sqrt{\lambda_{\min}(A)}\alpha} + \frac{100\sqrt{\alpha}C^{2}K}{\left(\lambda_{\min}(A)\right)^{3/2}}\).
Next, we choose $\alpha$ small enough:
\(\alpha \leq \left(\frac{\epsilon(\lambda_{\min}(A))^{3/2}}{200 C^2 K}\right)^2\) to ensure that the second term in the above inequality is within an $\epsilon$ boundary and $n$ large enough:
\(n = \frac{4}{\sqrt{\alpha\lambda_{\min}(A)}}\log\left(\frac{50 C^2 \Lambda}{\epsilon}\right),\)
to ensure that the first term is within an $\epsilon$ boundary. Finally,
\begin{align*}
    n = \max\Bigg(\frac{4}{\sqrt{\alpha\lambda_{\min}(A)}}\log\left(\frac{50 C^2 \Lambda}{\epsilon}\right), \frac{4}{\sqrt{\alpha\lambda_{\min}(A)}}\log\left(\frac{1}{\lambda_{\min}(A)\alpha}\right)\Bigg)
\end{align*}
gives the desired bound on the sample complexity for $\beta = 0$.
    
\section{Concluding Remarks}
\label{sec:conclusion}
In this work, we analyze the sample complexity of SHB and ASG and provide matching lower and upper bounds up to constants and logarithmic terms.
More importantly, we show that the same sample complexity bound can be obtained by standard SGD. Our work also calls into question some of the recent positive results in favour of momentum methods in the stochastic regime. We show that such improvements do not take into account all the terms involved in the error decomposition, or have major flaws. Although some other results do question the superiority of momentum methods in the stochastic regime, the assumptions and the setting that these works look at do not correspond to those in the positive results. We also emphasize that the negative results either only consider small step-sizes and momentum parameters not close to 1 or provide such results for specific instances in linear regression. In contrast, our work shows that SHB and ASG cannot obtain an improvement in terms of sample complexity (for small neighbourhoods) over SGD for the entire family of quadratic optimization and holds for all step-sizes and all momentum parameters in $[0,1]$.

\bibliographystyle{apalike}
\bibliography{References}

\begin{thebibliography}{}

\bibitem[Assran and Rabbat, 2020]{Assran}
Assran, M. and Rabbat, M. (2020).
\newblock On the convergence of nesterov’s accelerated gradient method in
  stochastic settings.
\newblock {\em Proceedings of the 37th International Conference on Machine
  Learning, PMLR}, 119:410--420.

\bibitem[Borkar, 2008]{Borkar_Book}
Borkar, V.~S. (2008).
\newblock {\em Stochastic Approximation: A Dynamical Systems Viewpoint}.
\newblock Cambridge University Press.

\bibitem[Can et~al., 2019]{zhu2}
Can, B., Gurbuzbalaban, M., and Zhu, L. (2019).
\newblock Accelerated linear convergence of stochastic momentum methods in
  {W}asserstein distances.
\newblock In Chaudhuri, K. and Salakhutdinov, R., editors, {\em Proceedings of
  the 36th International Conference on Machine Learning}, volume~97 of {\em
  Proceedings of Machine Learning Research}, pages 891--901. PMLR.

\bibitem[Cauchy, 1847]{cauchy}
Cauchy, L.~A. (1847).
\newblock Methode generale pour la resolution des systemes d'equations
  simultanees.
\newblock {\em C.R. Acad. Sci. Paris}, 25:536--538.

\bibitem[Cheng et~al., 2020]{cheng}
Cheng, X., Yin, D., Bartlett, P., and Jordan, M. (2020).
\newblock Stochastic gradient and {L}angevin processes.
\newblock In III, H.~D. and Singh, A., editors, {\em Proceedings of the 37th
  International Conference on Machine Learning}, volume 119 of {\em Proceedings
  of Machine Learning Research}, pages 1810--1819. PMLR.

\bibitem[Foucart, 2012]{foucart}
Foucart, S. (2012).
\newblock University lecture, m504, lecture 6: Matrix norms and spectral radii.

\bibitem[Horn and Johnson, 1990]{horn}
Horn, R. and Johnson, C. (1990).
\newblock {\em Matrix analysis (Corrected reprint of the 1985 original)}.
\newblock Cambridge University Press.

\bibitem[Hu et~al., 2009]{Chonghai}
Hu, C., Pan, W., and Kwok, J. (2009).
\newblock Accelerated gradient methods for stochastic optimization and online
  learning.
\newblock In Bengio, Y., Schuurmans, D., Lafferty, J., Williams, C., and
  Culotta, A., editors, {\em Advances in Neural Information Processing
  Systems}, volume~22. Curran Associates, Inc.

\bibitem[Jastrzębski et~al., 2018]{jas}
Jastrzębski, S., Kenton, Z., Arpit, D., Ballas, N., Fischer, A., Storkey, A.,
  and Bengio, Y. (2018).
\newblock Three factors influencing minima in {SGD}.

\bibitem[Kidambi et~al., 2018]{rahul}
Kidambi, R., Netrapalli, P., Jain, P., and Kakade, S.~M. (2018).
\newblock On the insufficiency of existing momentum schemes for stochastic
  optimization.
\newblock {\em CoRR}, abs/1803.05591.

\bibitem[Liu and Belkin, 2020]{liu}
Liu, C. and Belkin, M. (2020).
\newblock Accelerating sgd with momentum for over-parameterized learning.
\newblock In {\em International Conference on Learning Representations}.

\bibitem[Loizou and Richtárik, 2020]{Loizou}
Loizou, N. and Richtárik, P. (2020).
\newblock Momentum and stochastic momentum for stochastic gradient, newton,
  proximal point and subspace descent methods.
\newblock {\em Computational Optimization and Applications}, 77:653--710.

\bibitem[Mandt et~al., 2017]{mandt}
Mandt, S., Hoffman, M.~D., and Blei, D.~M. (2017).
\newblock Stochastic gradient descent as approximate bayesian inference.
\newblock {\em Journal of Machine Learning Research}, 18:1--35.

\bibitem[Mou et~al., 2020]{MJ}
Mou, W., Li, C.~J., Wainwright, M.~J., Bartlett, P.~L., and Jordan, M.~I.
  (2020).
\newblock On linear stochastic approximation: Fine-grained {P}olyak-{R}uppert
  and non-asymptotic concentration.
\newblock In Abernethy, J. and Agarwal, S., editors, {\em Proceedings of Thirty
  Third Conference on Learning Theory}, volume 125 of {\em Proceedings of
  Machine Learning Research}, pages 2947--2997. PMLR.

\bibitem[Nesterov, 1983]{nesterov}
Nesterov, Y. (1983).
\newblock A method of solving a convex programming problem with convergence
  rate $o\bigl(\frac1{k^2}\bigr)$.
\newblock {\em Soviet Mathematics Doklady}, 269:543--547.

\bibitem[Nesterov, 2005]{Nesterov05}
Nesterov, Y. (2005).
\newblock Smooth minimization of non-smooth functions.
\newblock {\em Math. Program.}, 103(1):127–152.

\bibitem[Nesterov, 2014]{nesterovbook}
Nesterov, Y. (2014).
\newblock {\em Introductory Lectures on Convex Optimization: A Basic Course}.
\newblock Springer Publishing Company, Incorporated, 1 edition.

\bibitem[Nitanda, 2014]{Nitanda2014}
Nitanda, A. (2014).
\newblock Stochastic proximal gradient descent with acceleration techniques.
\newblock In {\em NIPS}.

\bibitem[Polyak, 1964]{polyak_heavy_ball}
Polyak, B. (1964).
\newblock Some methods of speeding up the convergence of iteration methods.
\newblock {\em Ussr Computational Mathematics and Mathematical Physics},
  4:1--17.

\bibitem[Polyak, 1987]{polyakbook}
Polyak, B. (1987).
\newblock {\em Introduction to Optimization}.
\newblock Optimization Software, NY.

\bibitem[Qian, 1999]{qian}
Qian, N. (1999).
\newblock On the momentum term in gradient descent learning algorithms.
\newblock {\em Neural Netw.}, 12(1):145–151.

\bibitem[Sutskever et~al., 2013]{sutskever13}
Sutskever, I., Martens, J., Dahl, G., and Hinton, G. (2013).
\newblock On the importance of initialization and momentum in deep learning.
\newblock In Dasgupta, S. and McAllester, D., editors, {\em Proceedings of the
  30th International Conference on Machine Learning}, volume~28 of {\em
  Proceedings of Machine Learning Research}, pages 1139--1147, Atlanta,
  Georgia, USA. PMLR.

\bibitem[Yuan et~al., 2016]{ucla}
Yuan, K., Ying, B., and Sayed, A.~H. (2016).
\newblock On the influence of momentum acceleration on online learning.
\newblock {\em Journal of Machine Learning Research}, 17(192):1--66.

\bibitem[Zhang et~al., 2019]{nqm}
Zhang, G., Li, L., Nado, Z., Martens, J., Sachdeva, S., Dahl, G., Shallue, C.,
  and Grosse, R.~B. (2019).
\newblock Which algorithmic choices matter at which batch sizes? insights from
  a noisy quadratic model.
\newblock In Wallach, H., Larochelle, H., Beygelzimer, A., d\textquotesingle
  Alch\'{e}-Buc, F., Fox, E., and Garnett, R., editors, {\em Advances in Neural
  Information Processing Systems}, volume~32. Curran Associates, Inc.

\end{thebibliography}

\newpage
\appendix

\begin{center}
    \vspace{50mm}
    \Huge\textbf{Appendix}
\end{center}
\vspace{-2mm}
\par\noindent\rule{\textwidth}{0.1pt}

\section{Comparison with recent works}
\label{App_A}
\subsection{Comparison with Mou et al.(2020)}
\label{App_A1}
Claim 1 in (p. 20, \citet{MJ}) analyzes the asymptotic covariance of the heavy ball momentum algorithm (with Polyak averaging) and claims a correction term that satisfies:
\[\mbox{Tr}(L_{\eta}) \lesssim \mathcal{O}(\eta\frac{\kappa^2(U)}{\lambda_{min}(A)^{3/2}})\]
where \(\tilde{A} = \begin{pmatrix}
    0 & I_{d}\\
    -\bar{A} & \alpha I_{d}+\eta \bar{A}
\end{pmatrix} = UDU^{-1}\) as in the decomposition of $\tilde{A}$ in Lemma 1 of \cite{MJ} and $\kappa(U) = \|U\|_{op}\|U^{-1}\|_{op}$.

However in the proof of claim 1, we are not sure if the following bound holds, since the matrix $\tilde{A}$ is not symmetric:
\[Tr(\tilde{A}^{-1}\mathbb{E}(\tilde{\Xi}_{A}\Lambda_{\eta}^{*}(\tilde{\Xi}_{A})^{T})(\tilde{A}^{-1})^{T}) \leq (\min_{i}|\lambda_{i}(\tilde{A})|)^{-2}(1 + \eta^{2}) v_{A}^{2}\mathbb{E}_{\pi_{\eta}}\|x_{t} - x^{*}\|_{2}^{2}
\]

Our calculation points towards the following bound:
\[\mbox{Tr}(L_{\eta}) \lesssim \mathcal{O}(\eta\frac{\kappa^2(U)}{\lambda_{min}(A)^{5/2}}).\]
We outline the proof for the uni-variate case, when $\bar{A} = \lambda$ for some $\lambda>0$. Then, 
\[\tilde{A} = \begin{pmatrix}
     0 & 1\\
    -\lambda & \alpha +\eta \lambda
\end{pmatrix},
\mbox{ and } 
\tilde{A}^{-1} = \frac{1}{\lambda}\begin{pmatrix}
     \alpha +\eta \lambda & -1\\
    \lambda & 0
\end{pmatrix}.\]
Observe that \(\displaystyle\tilde{A}^{-1}(\tilde{A}^{-1})^T = \frac{1}{\lambda^2}
\begin{bmatrix}
    1 + (\alpha+\eta\lambda)^2 & \lambda(\alpha + \eta\lambda)\\
    \lambda(\alpha + \eta\lambda) & \lambda^2
\end{bmatrix}\)
and therefore $\displaystyle Tr(\tilde{A}^{-1}(\tilde{A}^{-1})^T) = \mathcal{O}\left(\frac{1}{\lambda^2}\right)$.
Using this we have, \[Tr(\tilde{A}^{-1}\mathbb{E}(\tilde{\Xi}_{A}\Lambda_{\eta}^{*}(\tilde{\Xi}_{A})^{T})(\tilde{A}^{-1})^{T}) \leq \mathcal{O}(\frac{1}{\lambda^2})Tr(\mathbb{E}(\tilde{\Xi}_{A}\Lambda_{\eta}^{*}(\tilde{\Xi}_{A})^{T})\]
\begin{equation}
\label{mou_bound}
    \lesssim \mathcal{O}(\eta\frac{\kappa^2(U)}{\lambda^{5/2}})
\end{equation}
The second inequality follows as in \cite{MJ}. Next we analyze the dependence of $\kappa^2(U)$ on $\lambda$. Again for simplicity we consider the uni-variate case where $\bar{A} = \lambda$. Let
\(\displaystyle\tilde{A} = \begin{pmatrix}
     0 & 1\\
    -\lambda & \alpha +\eta \lambda
\end{pmatrix},\) be diagonalizable. Therefore,
\[
    \tilde{A} = U 
    \begin{bmatrix}
        \mu_{+} & 0\\
        0 & \mu_{-}
    \end{bmatrix}
    U^{-1},
\]
where $\mu_{+}$ and $\mu_{-}$ are the eigenvalues of $\tilde{A}$.
Let \(U = \begin{bmatrix}
    u_{1} & u_{2}\\
    u_{3} & u_{4}
\end{bmatrix}\). We therefore have,
\[
    \begin{bmatrix}
     0 & 1\\
    -\lambda & \alpha +\eta \lambda
    \end{bmatrix}
    \begin{bmatrix}
    u_{1} & u_{2}\\
    u_{3} & u_{4}
\end{bmatrix} = 
\begin{bmatrix}
    u_{1} & u_{2}\\
    u_{3} & u_{4}
\end{bmatrix}
\begin{bmatrix}
        \mu_{+} & 0\\
        0 & \mu_{-}
    \end{bmatrix}.
\]
Solving the system of equations, we get:
\[
U = \begin{bmatrix}
        1 & 1\\
        \mu_{+} & \mu_{-}
    \end{bmatrix}
    \mbox{ and }
    U^{-1} = \frac{1}{\mu_{+} - \mu_{-}}
    \begin{bmatrix}
        1 & 1\\
        \mu_{+} & \mu_{-}
    \end{bmatrix}
\]
Now, $\mu_{+} - \mu_{-} = \sqrt{(\alpha + \eta\lambda)-4\lambda}$. Using the choice of $\alpha = \sqrt{\lambda}$ as in \cite{MJ}, we have:
\[
\mu_{+} - \mu_{-} = \sqrt{ \lambda + \eta^2\lambda^2 + 2\alpha\eta\lambda - 4\lambda}
\]
\[
= \sqrt{\lambda}\sqrt{\eta^2\lambda+2\eta\sqrt{\lambda} - 3}
\]
For $\lambda << 1$ (which is the case where the momentum algorithm is claimed to improve the mixing time in \cite{MJ}), \(\mu_{+} - \mu_{-} \geq \mathcal{O}(\sqrt{\lambda})\).
As in proof of Lemma \ref{c_hat_lemma} in Appendix \ref{C2} and using the fact that $\displaystyle\|U\|_{op}\|U^{-1}\|_{op} = \sigma_{max}(U)\sigma_{max}(U^{-1})$, we can show that $\kappa^2(U) \leq \mathcal{O}(\frac{1}{\lambda})$. Combining with \eqref{mou_bound}, we have:
\[
\mbox{Tr}(L_{\eta}) \lesssim \mathcal{O}(\eta\lambda^{-7/2}).
\]
A similar analysis can be carried for the multivariate case to show that 
\[\mbox{Tr}(L_{\eta}) \lesssim \mathcal{O}(\eta\lambda_{min}(\bar{A})^{-7/2})\]
The correction term in SGD is $\mathcal{O}(\eta\lambda_{min}(\bar{A})^{-3})$ (See \cite{MJ}, Claim 1). 
The stationary distribution for the momentum algorithm is larger than that of SGD when $\lambda_{min}(A) << 1$. Indeed if we enforce that the two asymptotic covariances are of the same size by choosing the step-size for momentum iterate $\eta^{m}$ in terms of the step-size of SGD, i.e.,
\[
\mathcal{O}(\eta\frac{1}{\lambda_{min}(\bar{A})^3}) = \mathcal{O}(\eta^{m} \frac{1}{\lambda_{min}(\bar{A})^{7/2}}),
\]
then we must choose $\eta^{m} = \mathcal{O}(\eta\sqrt{\lambda_{min}(\bar{A})})$. 
In Appendix C.1. of \cite{MJ}, the mixing time of momentum iterate is shown to be $\displaystyle\mathcal{O}(\frac{1}{\eta^m\sqrt{\lambda_{min}(\bar{A})}})$, while the mixing time of SGD is $\displaystyle\mathcal{O}(\frac{1}{\eta\lambda_{min}(\bar{A})})$. When we choose $\eta^{m} = \mathcal{O}(\sqrt{\lambda_{min}(\bar{A})}\eta)$, then the mixing time of momentum algorithm turns out to be the same as SGD.  This behaviour is identical to what we observe in Theorem \ref{Theorem_OTS_mom}, where if we choose the same step-size then there is improvement by a square root factor. 

\subsection{Comparison with SHB, Can et al. (2019)}
\label{App_A2}
For strongly convex quadratic functions of the form:
\[
f(x) = \frac{1}{2}x^TQx + a^Tx + b,
\]
where $x\in \mathbb{R}^d$, $Q\in R^{d\times d}$ is p.s.d, $a\in \mathbb{R}^d$, $b\in R$ and $\mu I_d \preceq Q \preceq LI_d$, \cite{zhu2} shows acceleration in Wasserstein distance by a factor of $\displaystyle\sqrt{\kappa} = \sqrt{\frac{L}{\mu}}$. The trace of the asymptotic covariance matrix $X_{HB}$ is given by (See Appendix C.2 of \cite{zhu2}):
\[
\mbox{Tr}(X_{HB}) = \sum_{i=1}^{d}\frac{2\alpha(1+\beta)}{(1-\beta)\lambda_{i}(2+2\beta-\alpha\lambda_{i})},
\]
where, $\alpha$ is the step-size, $\beta$ is the momentum parameter and $\lambda_{i}$ is the $i^{th}$ eigen-value of $Q$.
We show that the asymptotic covariance matrix is worse compared to when no momentum is used, i.e., $\beta = 0$ and the optimial step-size $\alpha = \frac{2}{\mu+L}$ is used. Substituting these values for $\beta$ and $\alpha$ we get:
\begin{align*}
    \mbox{Tr}(X_{\beta=0}) &= \sum_{i=1}^{d}\frac{2\frac{2}{\mu + L}}{\lambda_{i}(2-\frac{2}{\mu + L}\lambda_{i})}\\
    &= \sum_{i=1}^{d}\frac{2\frac{2}{\mu + L}}{\lambda_{i}\frac{2}{\mu+L}(\mu + L -\lambda_{i})}\\
    &= \sum_{i=1}^{d}\frac{2}{\lambda_{i}(\mu + L -\lambda_{i})}
\end{align*}
To compute the size of the stationary distribution with the iterates of heavy ball we set the step-size $\alpha =\displaystyle \frac{4}{(\sqrt{\mu}+\sqrt{L})^2}$ and momentum parameter $\beta = \left(\frac{\sqrt{L}-\sqrt{\mu}}{\sqrt{L}+\sqrt{\mu}}\right)^2$ and get:
\begin{align*}
    \mbox{Tr}(X_{HB}) &= \sum_{i=1}^{d}\frac{2\displaystyle\left(\frac{4}{(\sqrt{\mu} + \sqrt{L})^2}\right)(1+\beta)}{(1-\beta)\lambda_{i}\left(2+2\left(\frac{\sqrt{L}-\sqrt{\mu}}{\sqrt{L}+\sqrt{\mu}}\right)^2-\frac{4}{(\sqrt{\mu} + \sqrt{L})^2}\lambda_{i}\right)}\\
    &= \sum_{i=1}^{d}\frac{2\left(\displaystyle\frac{4}{(\sqrt{\mu} + \sqrt{L})^2}\right)(1+\beta)}{(1-\beta)\lambda_{i}\left(\displaystyle\frac{2(\sqrt{L}+\sqrt{\mu})^2 + 2(\sqrt{L}-\sqrt{\mu})^2-4\lambda_{i}}{(\sqrt{L}+\sqrt{\mu})^2}\right)}\\
    &= \sum_{i=1}^{d}\frac{2\left(\displaystyle\frac{4}{(\sqrt{\mu} + \sqrt{L})^2}\right)(1+\beta)}{(1-\beta)\lambda_{i}\left(\displaystyle\frac{4}{(\sqrt{\mu} + \sqrt{L})^2}\right)(\mu + L -\lambda_{i})}\\
    &= \sum_{i=1}^{d}\frac{2(1+\beta)}{(1-\beta)\lambda_{i}(\mu + L -\lambda_{i})}\\
    &= \sum_{i=1}^{d}\frac{2}{\lambda_{i}(\mu + L -\lambda_{i})}\left(\frac{1+\left(\frac{\sqrt{L} - \sqrt{\mu}}{\sqrt{L} + \sqrt{\mu}}\right)^2}{1-\left(\frac{\sqrt{L} - \sqrt{\mu}}{\sqrt{L} + \sqrt{\mu}}\right)^2}\right)\\
    &= \sum_{i=1}^{d}\frac{2}{\lambda_{i}(\mu + L -\lambda_{i})}\frac{L+\mu}{4\sqrt{\mu L}}\\
    &= \sum_{i=1}^{d}\frac{2}{\lambda_{i}(\mu + L -\lambda_{i})}\frac{1}{2}\left(\sqrt{\kappa}+\frac{1}{\sqrt{\kappa}}\right)\\
    &= \sum_{i=1}^{d}\frac{2}{\lambda_{i}(\mu + L -\lambda_{i})}\mathcal{O}(\sqrt{\kappa})\\
    &= \mbox{Tr}(X_{\beta=0})\mathcal{O}(\sqrt{\kappa})
\end{align*}
The above calculation shows that the size of the asymptotic covariance matrix of SHB is worse by a factor of $\mathcal{O}(\sqrt{\kappa})$.

\subsection{Comparison with Assran and Rabbat (2020)}
\label{App_A4}
Theorem 1 of \cite{Assran} is stated for a constant $\epsilon$ (not to be confused with $\epsilon$ in the current paper) 
and uses $C_{\epsilon}$ (a term that depends on $\epsilon$). However in Corollary 1.1, the result is stated for a decreasing sequence $\epsilon_k$ for which $C_{\epsilon}$ is undefined (equation (26)).

Using the expressions provided in Appendix of \cite{Assran}, we find that the second term in Corollary 1.1 (variance bound) is in fact of order $Q^{3/2}$ instead of $Q^{1/2}$. Using equation (29) and the display below it, we see that for ASG, $\| A^k \| \leq k \big(1- \frac{1}{\sqrt{Q}}\big)^{k+1}$. 
It follows that $$\sum_{k=0}^{n-1} \| A^{k}\|^2 \leq \sum_{k=0}^{n-1} k^2 \Big(1-\frac{1}{\sqrt{Q}}\Big)^{2k+2} = O\Bigg(\frac{1}{\big(1-(1-\frac{1}{\sqrt{Q}})^2\big)^{3}}\Bigg) = O(Q^{3/2}).$$ 

However, for SGD, $$\sum_{k=0}^{n-1} \| A^{k}\|^2 \leq \sum_{k=0}^{n-1} \Big(1-\frac{2}{\sqrt{Q-1}}\Big)^{2k} = O\Bigg(\frac{1}{1-(1-\frac{2}{Q-1})^2} \Bigg) = O(Q).$$

As a result, the variance bound for ASG is $O(\sigma^2 Q^{3/2})$, while the variance bound for SGD is $O(\sigma^2 Q)$. Though the bias term in their result improves by a factor of $\sqrt{Q}$, the variance term also worsens by a factor of $\sqrt{Q}$.

\section{Proof of Theorem \ref{Lower_bound}}
\label{App_thm3}

Observe that
\begin{align*}
    \|P^{n}\tilde{X}_{0}\|^2 &= \|P^{n}\|^2\|\tilde{X}_{0}\|^2 \mbox{ for some } \tilde{X}_0 \\
    &\geq \rho(P)^{2n} \|\tilde{X}_0\|^2\\
\end{align*}
Choose such an $\tilde{X}_{0}$ with $\|\tilde{X}_{0}\|\geq1$. If $\|P^n \tilde{X}_{0}\|^2 > \epsilon$, $\mathbb{E}[\|\tilde{X}_{n}\|^2] > \epsilon$. Therefore we assume $\|P^n \tilde{X}_{0}\|^2 \leq \epsilon$ which implies $\rho(P)^{2n} \leq \epsilon$ and $\|P^n e_{1}\|^2 \leq \epsilon$.
\begin{customlemma}{1}
For all $\alpha>0$, $\eta \in [0,(1-\sqrt{\alpha \lambda})^2)$, if $\|P^n \tilde{X_0} \|^2 < \epsilon$ for small enough $\epsilon>0$ then the variance term $\alpha^2 K \sum_{j=0}^{n-1} \Big\|P^{j}e_{1}\Big\|^2$ is bounded by $$\alpha^2 K \sum_{j=0}^{n-1} \Big\|P^{j}e_{1}\Big\|^2 \geq \frac{\alpha^2 K}{2(1-\mu_+^2)(1-\mu_-^2)(1-\eta)}.$$
\end{customlemma}
\begin{proof}
We see that
\begin{align*}
    \sum_{j=0}^{n-1} \Big\|P^{j}e_{1}\Big\|^2 &\geq \sum_{j=0}^{n-1} \left(\frac{\mu_+^j-\mu_-^j}{\mu_+-\mu_-}\right)^2 \\
    &= \frac{1}{(\mu_+-\mu_-)^2} \sum_{j=0}^{n-1}(\mu_+^{2j}+\mu_-^{2j}-2(\mu_+\mu_-)^{j})\\
    &= \frac{1}{(\mu_+-\mu_-)^2} \left(\sum_{j=0}^{n-1} \mu_+^{2j}+ \sum_{j=0}^{n-1} \mu_-^{2j}-2 \sum_{j=0}^{n-1} (\mu_+\mu_-)^{j}\right) \\
    &= \frac{1}{(\mu_+-\mu_-)^2} \left( \frac{1-\mu_+^{2n}}{1-\mu_+^2}+ \frac{1-\mu_-^{2n}}{1-\mu_-^2}-2 \frac{1-(\mu_+\mu_-)^{n}}{1-\mu_+\mu_-}\right) 
\end{align*}

The sum of fractions in the bracket can be expressed as a single fraction with denominator $(1-\mu_+^2)(1-\mu_-^2)(1-\mu_+\mu_-)$. The numerator turns out to be:
\begin{gather*}
    num = (\mu_+-\mu_-)^2+\mu_+\mu_-(\mu_+-\mu_-)^2+\mu_+\mu_-(\mu_+^n-\mu_-^n)^{2}+(\mu_+\mu_-)^2(\mu_+^{n-1}-\mu_-^{n-1})^{2} - \\(\mu_+\mu_-)^3(\mu_+^{n-1}-\mu_-^{n-1})^2 - 2(\mu_+\mu_-)^n(\mu_+-\mu_-)^2-(\mu_+^n-\mu_-^n)^2.
\end{gather*}

First consider the case $\eta \in [0, (1-\sqrt{\alpha\lambda})^2))$. The eigen values $\mu_{+}$ and $\mu_{-}$ in this case are real. Therefore all the terms in the numerator and the denominator are real.
Next consider the case $\eta \in ((1-\sqrt{\alpha\lambda})^2,1]$. In this case the eigen values $\mu_+$ and $\mu_-$ are complex. Let $\mu_+ = \rho(P)(\cos\omega + i\sin\omega)$ and $\mu_- = \rho(P)(\cos\omega - i\sin\omega)$. Observe that 
\begin{align*}
    (\mu_+ - \mu_-)^2 &= (\rho(P)(\cos\omega + i\sin\omega) - \rho(P)(\cos\omega - i\sin\omega))^2 = -4\rho(P)^2\sin^2\omega,\\
    (\mu_+^m - \mu_-^m)^2 &= -4\rho(P)^{2m}\sin^2 m \omega,\\
    \mu_+\mu_- &= \rho(P)^2(\sin^2\omega + \cos^2\omega)=\rho(P)^2 = \eta.
\end{align*}
Therefore, all the terms in the numerator and the denominator are again real. Now observe that,
\begin{align*}
    &\mu_+\mu_-(\mu_+-\mu_-)^2 \geq 0,\\
    & (\mu_+\mu_-)^2(\mu_+^{n-1}-\mu_-^{n-1})^{2} -
    (\mu_+\mu_-)^3(\mu_+^{n-1}-\mu_-^{n-1})^2 \geq 0.
\end{align*}

The second inequality follows because $\mu_+ \mu_- = \frac{(1-\alpha \lambda +\eta)^2-((1-\alpha \lambda +\eta)^2-4\eta)}{4}=\eta$ and therefore $(\mu_+\mu_-)^2 \geq (\mu_+\mu_-)^3$. Next, 
\begin{align*}
     \mu_+\mu_-(\mu_+^n-\mu_-^n)^{2} &= \mu_+\mu_-(\mu_+-\mu_-)^2\Big(\sum_{i=0}^{n-1}\mu_+^{i}\mu_-^{n-1-i} \Big)^{2}\\
    & \geq \mu_+\mu_-(\mu_+-\mu_-)^2(\mu_+^{n-1}+\mu_-^{n-1} )^{2} \\
    & \geq 4(\mu_+\mu_-)^{n}(\mu_+-\mu_-)^2 \\
\end{align*}
The last inequality follows since $\mu_+^{n-1}+\mu_-^{n-1} \geq 2 (\mu_+\mu_-)^{\frac{n-1}{2}}$. It follows that $\mu_+\mu_-(\mu_+^n-\mu_-^n)^{2}-2(\mu_+\mu_-)^n(\mu_+-\mu_-)^2 \geq 0$. Also, note that $ \left(\frac{\mu_+^n-\mu_-^n}{\mu_+-\mu_-}\right)^2 \leq \|P^ne_1 \|^2 <\epsilon$ and therefore $(\mu_+^n-\mu_-^n)^2 \leq \epsilon (\mu_+-\mu_-)^2$ .
It follows that $num \geq (1-\epsilon)(\mu_+-\mu_-)^2$. Using these bounds, taking $\epsilon<=1/2$ and noting that $\displaystyle h(\eta, \alpha \lambda):= \frac{(1-\mu_+^2)(1-\mu_-^2)(1-\mu_+\mu_-)}{(\alpha \lambda)^2}$, we get
\begin{align}
    \alpha^2 K \sum_{j=0}^{n-1} \Big\|P^{j}e_{1}\Big\|^2
    & \geq \frac{\alpha^2K}{(\mu_+-\mu_-)^2} \left( \frac{1-\mu_+^{2j}}{1-\mu_+^2}+ \frac{1-\mu_-^{2j}}{1-\mu_-^2}-2\nonumber \frac{1-(\mu_+\mu_-)^{j}}{1-\mu_+\mu_-}\right)\\
    \label{var_eq}
    & \geq\frac{\alpha^2K}{2(1-\mu_+^2)(1-\mu_-^2)(1-\mu_+\mu_-)}.
\end{align}
Finally, consider the case when $\eta = (1 - \sqrt{\alpha\lambda})^2$. In this case $\mu_+ = \mu_- = \mu$.

Let $P=S^{-1}JS$, where $J=\begin{bmatrix}
    \mu & 1 \\
    0 & \mu
\end{bmatrix}.$ Let $S=\begin{bmatrix}
    x & y \\
    w & z
\end{bmatrix}$. Since $\eta = (1-\alpha \lambda)^2$, we get $\mu = \sqrt{\eta} = (1-\alpha \lambda)$ and $1-\alpha \lambda + \eta = 2 \sqrt{\eta}$. We now solve for $S$. 

$$\begin{bmatrix}
    x & y \\
    w & z
\end{bmatrix} \begin{bmatrix}
    1-\alpha \lambda +\eta & -\eta \\
    1 & 0
\end{bmatrix} = \begin{bmatrix}
    \mu & 1 \\
    0 & \mu
\end{bmatrix} \begin{bmatrix}
    x & y \\
    w & z
\end{bmatrix}.$$
It follows that the following equations hold
\begin{align*}
    x(1-\alpha \lambda + \eta)+y&=\mu x +w \\
    w(1-\alpha \lambda + \eta)+z&=\mu w\\
    \eta x &= \mu y +z \\
    \eta w &=\mu z. 
\end{align*}
Solving the above equations with $w=1$, we get,
\[
S = \begin{bmatrix}
    0 & 1\\
    1 & -\sqrt{\eta}
\end{bmatrix}
\]
\begin{align*}
P^j &= S^{-1} J^j S\\
&= \begin{bmatrix}
    \sqrt{\eta} & 1 \\
    1 & 0
\end{bmatrix}
\begin{bmatrix}
    \mu^j & n\mu^{j-1} \\
    0 & \mu^j
\end{bmatrix}
\begin{bmatrix}
    0 & 1\\
    1 & -\sqrt{\eta}
\end{bmatrix}\\
& = \begin{bmatrix}
    \mu^{j+1} & (j+1)\mu^j\\
    \mu^{n} & j \mu^{j-1}
\end{bmatrix}
\begin{bmatrix}
    0 & 1\\
    1 & -\mu
\end{bmatrix}\\
&= \begin{bmatrix}
    (j+1)\mu^j & -j\mu^{j+1}\\
    j \mu^{j-1} & -(j-1) \mu^j
\end{bmatrix}
\end{align*}
Observe that,
\begin{align*}
    \sum_{j=0}^{n-1} \|P^{j}e_{1}\|^2 &= \sum_{j=0}^{n-1} (j+1)^2\mu^2j + j^2\mu^{2(j-1)} \geq  \sum_{j=0}^{n-1} (j\mu^{j-1})^2\\
    & = \lim_{\mu_+ \rightarrow \mu_-}  \sum_{j=0}^{n-1}\left(\frac{\mu_+^j - \mu_-^j}{\mu_+-\mu_-}\right)^2 \\
    & \geq  \lim_{\mu_+ \rightarrow \mu_-} \frac{1}{2(1-\mu_+^2)(1-\mu_-^2)(1-\mu_+\mu_-)},
\end{align*}
where the last inequality follows from \eqref{var_eq}. Therefore,
$$\alpha^2 K \sum_{j=0}^{n-1} \Big\|P^{j}e_{1}\Big\|^2 \geq \frac{\alpha^2 K}{2(1-\mu^2)(1-\mu^2)(1-\eta)}.$$
\end{proof}
\begin{customlemma}{2}
For all $\alpha >0$, $\eta \in [0,(1-\sqrt{\alpha \lambda})^2)$,  if $\|P^n \tilde{X_0} \|^2 < \epsilon$ for small enough $\epsilon>0$, then $\alpha^2 K \sum_{j=0}^{n-1} \Big\|P^{j}e_{1}\Big\|^2 \geq \frac{K}{16 \lambda^2}(1-\rho(P)).$
\end{customlemma}
\begin{proof}
\label{C4}
Before proceeding to the proof, we introduce a new function 
\begin{align*}
\label{h-def}
    h(\eta, \alpha \lambda) \triangleq \frac{(1-\mu_+^2)(1-\mu_-^2)(1-\eta)}{(\alpha \lambda)^2}    
\end{align*}
to simplify calculations. 
First we consider the case $\eta \in [0,(1-\sqrt{\alpha\lambda})^2]$ when $\mu_+$ and $\mu_-$ are real.
Since $\mu_+$ and $\mu_-$ are only functions of $\alpha \lambda$ and $\eta$, $h$ is well-defined.

The proof of the Lemma follows by showing $h(\eta, \alpha \lambda)(1-\rho(P)) \leq 8$. Note that
$1-\mu_+^2 = (1-\mu_+)(1+\mu_+) \leq 2(1-\mu_+)$ since $\mu_+ \leq 1$. Similarly, $(1-\mu_-^2) \leq 2(1-\mu_-)$. Thus, $h(\eta, \alpha \lambda)(1-\mu_+) \leq \frac{4 (1-\mu_+)^2(1-\mu_-)(1-\mu_+\mu_-)}{(\alpha \lambda)^2}$. 
Since $\mu_+ + \mu_- = 1-\alpha \lambda + \eta$, it follows that $(1-\mu_+)(1-\mu_-)=1+\mu_ +  \mu_- -\mu_+ -\mu_-= 1-(1-\alpha \lambda + \eta)+\eta=\alpha\lambda$. Thus,
\begin{align*}
    h(\eta, \alpha \lambda)(1-\mu_+) &\leq \frac{4(1-\mu_+)(1-\eta)}{\alpha\lambda}\\
    & = \frac{4(1-\eta)}{\alpha\lambda}\Bigg(1-\frac{( 1 + \eta - \lambda \alpha) + \sqrt{(\lambda \alpha - 1 - \eta)^2 - 4\eta}}{2}\Bigg) \triangleq g(\eta, \alpha\lambda).
\end{align*}

We see that 
\begin{align*}
    \frac{\partial g(\eta, \alpha\lambda)}{\partial \eta} = \frac{2\Big(\sqrt{(1-\alpha\lambda+\eta)^2 - 4\eta}+\eta-1\Big)\Big(\sqrt{(1-\alpha\lambda+\eta)^2 - 4\eta}+\eta-\alpha\lambda-1\Big)}{\alpha\lambda\sqrt{(1-\alpha\lambda+\eta)^2 - 4\eta}}
\end{align*}

Observe that the denominator in the above expression is positive. We consider the following two cases:  (i) $\alpha\lambda \geq 1$ and (ii) $\alpha\lambda < 1$. When $\alpha\lambda \geq 1$, we can directly bound $g(\eta,\alpha\lambda)=\frac{4(1-\mu_+)(1-\eta)}{\alpha\lambda}$. Since $\mu_+ \geq -1$ and $\eta \geq 0$, we get $g(\eta,\alpha\lambda) \leq 8$.

Now consider the case  $\alpha\lambda < 1$. We have that, 
\begin{align*}
    \sqrt{(1-\alpha\lambda+\eta)^2 - 4\eta}+\eta-1 &= \sqrt{(1-\eta)^2+(\alpha\lambda)^2 - 2\alpha\lambda(1+\eta)} - (1-\eta)
\end{align*}
When $\alpha\lambda < 1$, $(\alpha\lambda)^2 - 2\alpha\lambda(1+\eta)<0$ for all $\eta \geq 0$, thus $\Big(\sqrt{(1-\alpha\lambda+\eta)^2 - 4\eta}+\eta-\alpha\lambda-1\Big) \leq \Big(\sqrt{(1-\alpha\lambda+\eta)^2 - 4\eta}+\eta-1\Big) < 0$. This implies that the numerator of the partial derivative is also positive and we have that $\displaystyle\frac{\partial g(\eta, \alpha\lambda)}{\partial \eta} > 0$. Since the partial derivative is positive $g(\eta, \alpha\lambda)$ is an increasing function of $\eta$ and thus the maximum is achieved at $\eta=(1-\sqrt{\alpha \lambda})^2$ and is given by:
\begin{align*}
    g(\eta, \alpha\lambda) &\leq \frac{4(1-\mu_+)(1-\eta)}{\alpha\lambda}\\
    & = \frac{4(1-\rho(P))(1-\eta)}{\alpha\lambda}\\
    & = \frac{4(1-(1-\sqrt{\alpha\lambda}))(1-\eta)}{\alpha\lambda}\\
    & = \frac{4\sqrt{\alpha\lambda}(2\sqrt{\alpha\lambda} - \alpha\lambda)}{\alpha\lambda}\\
    &\leq 8
\end{align*}
Next, for $\eta = ((1-\sqrt{\alpha\lambda})^2,1]$, the eigen values $\mu_+$ and $\mu_-$ are complex. We have,
\begin{align*}
    h(\eta,\alpha\lambda) = \frac{(1-\mu_+^2)(1-\mu_-^2)(1-\eta)}{(\alpha\lambda)^2}.
\end{align*}
First we show that $(1-\mu_+^2)(1-\mu_-^2) \leq 4 (1-\mu_+)(1-\mu_-)$. Notice,
\begin{align*}
    (1-\mu_+^2)(1-\mu_-^2) &= (1-\mu_+)(1+\mu_+)(1-\mu_-)(1+\mu_-)\\
    & = (1 - \mu_+)(1 - \mu_)(1+ \mu_+ + \mu_- + \mu_+\mu_-)\\
    & = (1 - \mu_+)(1 - \mu_-)(1 + (1-\alpha\lambda+\eta) + \eta) \\
    & \leq 4 (1 - \mu_+)(1 - \mu_-).
\end{align*}
It follows that, 
\begin{align*}
    h(\eta,\alpha\lambda)(1-\rho(P)) &\leq \frac{4(1-\rho(P))(1-\eta)}{\alpha\lambda}\\
    &= \frac{4(1-\sqrt{\eta})(1-\eta)}{\alpha\lambda} \triangleq l(\eta,\alpha\lambda)
\end{align*}
Observe that $l(\eta,\alpha\lambda)$ is a decreasing function of $\eta$ and therefore, the infimum is attained for $\eta = (1-\sqrt{\alpha\lambda})^2$. For this choice of $\eta$,
\begin{align*}
    l(\eta,\alpha\lambda) &= \frac{4(1 - (1-\sqrt{\alpha\lambda}))(1-(1-\sqrt{\alpha\lambda})^2)}{\alpha\lambda}
    \leq 8
\end{align*}
\end{proof}
\begin{customlemma}{3}
 If $n=\frac{K}{64\epsilon \lambda^2}\log(\frac{\|\tilde{X_0}\|^2}{\epsilon})$ and $\|P^n\tilde{X}_{0} \|^2 < \epsilon$, then $1-\mu_+ \geq \frac{16\epsilon \lambda^2}{K}$.
\end{customlemma}
\begin{proof}
From $\|P^n\tilde{X}_0\| \leq \epsilon$ we have $\rho(P)^{2n} \leq \epsilon$. 
Then,
\begin{gather*}
     \epsilon \geq \rho(P)^{2n} 
     \geq e^{-2n\frac{1-\rho(P)}{\rho(P)}}
\end{gather*}
Thus, we have $e^{-2n\frac{1-\rho(P)}{\rho(P)}}<\epsilon$ and $n = \frac{K}{64 \epsilon \lambda^2}\log(\frac{1}{\epsilon}) \geq \frac{\rho(P)}{2(1-\rho(P))} \log(\frac{1}{\epsilon})$. Thus, by re-arranging, and choosing $\epsilon$ such that $1 \leq \frac{K}{32\epsilon \lambda^2}$, we have

\begin{align}
    \frac{1}{1-\rho(P)} \leq \frac{K}{32\epsilon \lambda^2}+1 \leq \frac{K}{16 \epsilon \lambda^2}.
\end{align}

\end{proof}

\section{Proof of Theorem \ref{Theorem_OTS_mom}}
\label{App_thm1}
We prove the theorem separately for $\eta = 0$ (corresponding to SGD), $\beta = 0$ (corresponding to SHB) and $\beta = 1$ (corresponding to ASG).

\textbf{Case-1: $\eta = 0$ (SGD)}

The LSA-M iterate in \eqref{sgd-m} with $\eta = 0$ corresponds to:
\begin{gather}
    \label{transformedSA}
    x_{n+1} - x^* = x_{n}-x^* + \alpha(Ax^* - Ax_{n}+M_{n+1})\\
    \label{transformedSAMom}
    x_{n+1} - x^* = x_{n}-x^* + \alpha(Ax^* - Ax_{n}+M_{n+1}) + \eta((x_{n}-x^{*}) - (x_{n-1}-x^{*}))
\end{gather}
Let $\tilde{x}_{n} = x_{n} - x^{*}$. Then, equation \eqref{transformedSA} can be rewritten as:
\begin{equation*}
    \begin{split}
        \Tilde{x}_{n} &= \Tilde{x}_{n-1} + \alpha(-A\Tilde{x}_{n-1} + M_{n})
        = (I - \alpha A)\Tilde{x}_{n-1} + \alpha M_{n}\\
        &= (I - \alpha A)^{n}\Tilde{x}_{0} + \alpha\sum_{i=0}^{n-1}[(I-\alpha A)^{n-1-i} M_{i+1}]
    \end{split}
\end{equation*}
Taking the square of the norm on both sides of the above equation, we obtain
\begin{equation*}
    \begin{split}
        \|\Tilde{x}_{n}\|^2 &= \|(I-\alpha A)^n\Tilde{x}_{0}\|^{2} + 2\alpha\left((I-\alpha A)^{n}\Tilde{x}_{0}\right)^T\left(\sum_{i=0}^{n-1}(I-\alpha A)^{n-1-i}M_{i+1}\right)\\
        & + \alpha^2 \sum_{i=0}^{n-1}\sum_{j=0}^{n-1}((I-\alpha A)^{n-1-i} M_{i+1})^T ((I-\alpha A)^{(n-1-j)} M_{j+1}))
    \end{split}
\end{equation*}
Now we take expectation on both sides to obtain
\begin{equation*}
    \begin{split}
        \mathbb{E}[\|\Tilde{x}_{n}\|^2] 
        &\leq \|(I-\alpha A)^n\|^{2} \|\Tilde{x_{0}}\|^2 + 2\alpha\left((I-\alpha A)^{n}\Tilde{x}_{0}\right)^T\left(\sum_{i=0}^{n}(I-\alpha A)^{(n-1-i)}\mathbb{E}[M_{i+1}]\right)\\
        & + \alpha^2 \sum_{i=0}^{n-1}\sum_{j=0}^{n-1}\mathbb{E}((I-\alpha A)^{(n-1-i)} M_{i+1})^T ((I-\alpha A)^{(n-1-j)} M_{j+1}))
    \end{split}
\end{equation*}
Now, from \textbf{Assumption \ref{A2}}, $\mathbb{E}[M_{i+1}] = \mathbb{E}[\mathbb{E}[M_{i+1}|\mathcal{F}_{i}]] = 0.$
Therefore the second term becomes 0. Next consider the term inside the double summation. First consider the case $i\neq j$. Without loss of generality, suppose $i<j$.
\begin{equation*}
    \begin{split}
        & \mathbb{E}\left[M_{i+1}^{T}\left((I-\alpha A)^{(n-1-i)}\right)^T (I-\alpha A)^{(n-1-j)} M_{j+1}\right]\\
        & =\mathbb{E}\left[\mathbb{E}\left[M_{i+1}^{T}\left((I-\alpha A)^{(n-1-i)}\right)^T (I-\alpha A)^{(n-1-j)} M_{j+1}\vert\mathcal{F}_{j}\right]\right]\\
        & = \mathbb{E}\left[M_{i+1}^{T}\left((I-\alpha A)^{(n-1-i)}\right)^{T}(I-\alpha A)^{(n-1-j)}\mathbb{E}[M_{j+1}|\mathcal{F}_{j}]\right]
        = 0
    \end{split}
\end{equation*}
The last equality follows from \textbf{Assumption \ref{A2}}. When $i=j$,
\begin{equation*}
    \begin{split}
        & \mathbb{E}\left[M_{i+1}^{T}\left((I-\alpha A)^{(n-1-i)}\right)^T (I-\alpha A)^{(n-1-i)} M_{i+1}\right]\\
        & \leq \mathbb{E}\left[\|(I-\alpha A)^{(n-1-i)}\|^{2}\mathbb{E}\left[\|M_{i+1}\|^{2}|\mathcal{F}_{i}\right]\right]
         \leq \|(I-\alpha A)^{(n-1-i)}\|^{2} K\left(1 + \mathbb{E}[\|\Tilde{x}_{i}\|^{2}]\right)
    \end{split}
\end{equation*}
Substituting the above values and using $\Lambda =||x_{0} - x^*||^{2}$ we get
\begin{equation*}
    \begin{split}
        \mathbb{E}[\|\Tilde{x_n}\|^2] &\leq \|(I-\alpha A)^{n}\|^2\Lambda + \alpha^2 K\sum_{i=0}^{n-1}\|(I-\alpha A)^{(n-1-i)}\|^2 (1 + \mathbb{E}[\|\Tilde{x}_{i}\|^2])
    \end{split}
\end{equation*}
We next use the following lemma to bound $\|(I-\alpha A)^i\|$.
\begin{customlemma}{4}
 Let, $M \in \mathbb{R}^{d\times d}$ be a matrix and $\lambda_{i}(M)$ denote the $i^{th}$ eigen-value of $M$. Then, $\forall \delta > 0$
\[\|M^{n}\|\leq C_{\delta}(\rho(M) + \delta)^n\]
where $\rho(M) = \max_{i}|\lambda_{i}(M)|$ is the spectral radius of $M$ and $C_{\delta}$ is a constant that depends on $\delta$. Furthermore, if the eigen-values of $M$ are distinct, then
\[\|M^{n}\|\leq C(\rho(M))^n\]
\end{customlemma}
\begin{proof}
See Appendix \ref{Lemma5_proof}.
\end{proof}
We let $\lambda_{\min}(A) = \min_{i}\lambda_{i}(A)$. Using \textbf{Assumption \ref{A1}} and Lemma \ref{norm_upper_bound}, we have
\begin{equation*}
    \begin{split}
        \mathbb{E}[\|\Tilde{x}_{n}\|^2] &\leq C^2(1-\alpha\lambda_{\min}(A))^{2n}\Lambda + C^2\alpha^2K\sum_{i=0}^{n-1}(1-\alpha\lambda_{\min}(A))^{2(n-1-i)}(1 + \mathbb{E}[\|\Tilde{x}_{i}\|^2])
    \end{split}
\end{equation*}
\begin{equation}
    \label{C_thm1}
    \mbox{where, \quad } C = \frac{\sqrt{d}}{\sigma_{\min}(S)\sigma_{\min}(S^{-1})},
\end{equation}
$S$ is the matrix in Jordan decomposition of $A$ and $\sigma_{\min}(S)$ is the smallest singular value of $S$. 
We define the sequence $\{U_{k}\}$ as below:
\[U_{k} = C^2(1-\alpha\lambda_{\min}(A))^{2k}\Lambda + C^2\alpha^2K\sum_{i=0}^{k-1}(1-\alpha\lambda_{\min}(A))^{2(k-1-i)}(1 + U_{i})\]
Observe that $\mathbb{E}[\|\Tilde{x}_{n}\|^2] \leq U_{n}$ and that the sequence $\{U_{k}\}$ satisfies
\begin{gather*}
    U_{k+1} = (1-\alpha\lambda_{\min}(A))^{2}U_{k} + C^{2}K\alpha^{2}(1 + U_{k});\quad
    U_{0} = C^{2}\Lambda
\end{gather*}
Therefore, we have 
\[U_{k+1} = \left((1-\alpha\lambda_{\min}(A))^{2} + C^{2}K\alpha^{2}\right)U_{k} + \alpha^{2}C^{2}K\]
To ensure that $(1-\alpha\lambda_{\min}(A))^{2} + C^{2}K\alpha^{2} \leq (1-\alpha\lambda_{\min}(A)/2)^2$, choose $\alpha$ as follows:
\begin{gather*}
    \alpha^{2}\lambda_{\min}(A)^{2} - 2\alpha\lambda_{\min}(A) + C^{2}K\alpha^{2} \leq \frac{\alpha^{2}\lambda_{\min}(A)^{2}}{4} - \alpha\lambda_{\min}(A)\\
    \mbox{ or } \alpha \leq \frac{\lambda_{\min}(A)}{\frac{3}{4}\lambda_{\min}(A)^{2} + C^{2}K}
\end{gather*}
\begin{equation*}
    \begin{split}
        U_{n} & \leq \left(1 - \frac{\alpha\lambda_{\min}(A)}{2}\right)^{2}U_{n-1} + \alpha^{2}C^{2}K \\
        & \leq \left(1 - \frac{\alpha\lambda_{\min}(A)}{2}\right)^{2n}U_{0} + \alpha^{2}C^{2}K\sum_{i=0}^{n-1} \left(1 - \frac{\alpha\lambda_{\min}(A)}{2}\right)^{2i}\\
        & \leq  \left(1 - \frac{\alpha\lambda_{\min}(A)}{2}\right)^{2n}U_{0} + \alpha^{2}C^{2}K \frac{1}{1-\left(1-\frac{\alpha\lambda_{min}(A)}{2}\right)^{2}}\\
        & \leq \left(1 - \frac{\alpha\lambda_{\min}(A)}{2}\right)^{2n}U_{0} + \alpha^{2}C^{2}K \frac{2}{\alpha\lambda_{min}(A)}\\
    \end{split}
\end{equation*}
We assume that $\alpha \leq \frac{1}{\lambda_{\min}(A)}$ and therefore $(1-\frac{\alpha\lambda_{\min}(A)}{2})^{2} \leq e^{-\alpha\lambda_{\min}(A)}$.
\begin{equation*}
    \begin{split}
        U_{n}\leq e^{-n\alpha\lambda_{\min}(A)}C^{2}\Lambda + \frac{2\alpha C^{2}K}{\lambda_{min}(A)}
    \end{split}
\end{equation*}
Choose $\alpha$ as below:
\[\alpha \leq \frac{\epsilon\lambda_{min}(A)}{4C^2K}\]
Then,
\[\frac{2\alpha C^{2}K}{\lambda_{min}(A)} \leq \frac{\epsilon}{2}
\Rightarrow 
\mathbb{E}[\|\Tilde{x}_{n}\|^2] \leq U_{n} \leq \frac{\epsilon}{2} + \frac{\epsilon}{2} = \epsilon,\]
when the sample complexity is:
\[n = \frac{1}{\alpha\lambda_{\min}(A)}\log\left(\frac{2C^2 \Lambda}{\epsilon}\right)\]

\textbf{Case-2: $\beta = 0$ (SHB)}

\label{App_thm2}
LSA-M iterate with $\beta = 0$ can be re-written as:
\[\Tilde{x}_{n+1} = (I-\alpha A)\Tilde{x}_{n} + \alpha(M_{n+1}) + \eta(\Tilde{x}_{n} - \Tilde{x}_{n-1})\]
This can be re-written as:
\[
\begin{bmatrix}
    \Tilde{x}_{n+1} \\
    \Tilde{x}_{n}      
\end{bmatrix}
=
\begin{bmatrix}
    I - \alpha A + \eta I & -\eta I   \\
    I & 0      
\end{bmatrix} 
\begin{bmatrix}
    \Tilde{x}_{n} \\
    \Tilde{x}_{n-1}      
\end{bmatrix}
+ \alpha
\begin{bmatrix}
    M_{n+1}\\
    0      
\end{bmatrix}
\]
Let us define 
\[\Tilde{X}_{n} \triangleq
\begin{bmatrix}
    \Tilde{x}_{n} \\
    \Tilde{x}_{n-1}      
\end{bmatrix}
, P \triangleq
\begin{bmatrix}
    I - \alpha A + \eta I & -\eta I   \\
    I & 0      
\end{bmatrix} 
\mbox{ and } W_{n} \triangleq
\begin{bmatrix}
    M_{n}\\
    0
\end{bmatrix}.
\]
Note that $\mathbb{E}[W_{n+1}|\mathcal{F}_{n}] = 0$ and $\mathbb{E}[\|W_{n+1}\|^{2}|\mathcal{F}_{n}] = \mathbb{E}[\|M_{n+1}\|^{2}|\mathcal{F}_{n}] \leq K(1 + \mathbb{E}[\|\Tilde{x}_{n}\|^{2}]) \leq K(1+ \mathbb{E}[\|\Tilde{X}_{n}\|^{2}])$. It follows that,
\begin{gather*}
    \Tilde{X}_{n} = P \Tilde{X}_{n-1} + \alpha W_{n}
    = P^n \Tilde{X}_{0} + \alpha \sum_{i=0}^{n-1}P^{n-1-i}W_{i+1}
\end{gather*}
The norm square of the above equation gives:
\begin{equation*}
    \begin{split}
        \|\Tilde{X}_{n}\|^{2} & = \|P^{n}\Tilde{X}_{0}\|^{2} + \alpha \left(P^{n}\Tilde{X}_{0}\right)^{T}\left(\sum_{i=0}^{n-1}P^{(n-1-i)}W_{i+1}\right) + \alpha\left(\sum_{i=0}^{n-1}P^{(n-1-i)}W_{i+1}\right)^{T}\left(P^{n}\Tilde{X}_{0}\right)\\
        & + \alpha^{2} \left(\sum_{i=0}^{n-1}P^{(n-1-i)}W_{i+1}\right)^{T}\left(\sum_{i=0}^{n-1}P^{(n-1-i)}W_{i+1}\right)
    \end{split}
\end{equation*}
Taking expectation on both sides as well as using the fact that $\mathbb{E}[W_{k+1}|\mathcal{F}_{k}] = 0$ and $\mathbb{E}[\|W_{k+1}\|^{2}|\mathcal{F}_{k}] \leq K(1 + \mathbb{E}[\|\Tilde{X}_{k}\|^{2}])$, we have
\[
\mathbb{E}[\|\Tilde{X}_{n}\|^2] \leq \|P^n\|^2 \|\Tilde{X}_{0}\|^2 + \alpha^2K\sum_{i=0}^{n-1}\|P^{n-1-i}\|^2(1 + \mathbb{E}[\|\Tilde{X}_{i}\|^{2}])
\]
Without loss of generality assume $\Tilde{x}_{-1} = \mathbf{0}$. Therefore, $||\Tilde{X}_{0}||^2 = ||\Tilde{x}_{0}||^2 = \Lambda$. As before, for a matrix $M$, let $\rho(M) = \max_{i} |\lambda_{i}(M)|$ denote the spectral radius of $M$.
Next, we compute $\rho(P)$. Consider the characteristic equation of P:
\[
det\left(
\begin{bmatrix}
    I - \alpha A + \eta I -\mu I & -\eta I   \\
    I & -\mu I      
\end{bmatrix} \right) = 0
\]
When $A_{21}$ and $A_{22}$ commute, we have the following formula for determinant of a block matrix (\citet{horn}):
\[
det\left(
\begin{bmatrix}
A_{11} & A_{12}\\
A_{21} & A_{22}
\end{bmatrix}\right) = det\left(A_{11}A_{22} - A_{12}A_{21}\right)
\]
Using this, the characteristic equation of $P$ simplifies to:
\begin{gather*}
    det(- \mu I + \alpha\mu A - \eta \mu I + \mu^2I + \eta I)  = 0
\end{gather*}
We note that when $\mu=0$, the LHS of the above equation becomes $\det(\eta I)$. Thus, $\mu=0$ can never be a solution of the characteristic equation of $P$ whenever $\eta \neq 0$. We now further simplify the characteristic equation of $P$ to a more convienient form:
$$det\left(A - I \left(\frac{\mu+\eta\mu-\mu^2-\eta}{\alpha\mu}\right)\right)=0$$
The only zeros of the characteristic equation of a matrix are its eigenvalues. Let $\lambda_{i}(A)$ be the eigen-value of $A$ with \(\lambda_{i}(A) = \frac{\mu + \eta \mu - \mu^2 - \eta}{\alpha\mu}\) so that
\begin{gather*}
    \mu^2 + \mu(\alpha\lambda_{i}(A)-1-\eta) + \eta = 0
\end{gather*}
The above is a quadratic equation in $\mu$ and the solution is given by
\begin{gather*}
    \mu = \frac{-(\lambda_{i}(A)\alpha - 1 - \eta) \pm \sqrt{(\lambda_{i}(A)\alpha - 1 - \eta)^2 - 4\eta}}{2}
\end{gather*}
When $(\lambda_{i}(A)\alpha - 1 - \eta)^2 - 4\eta \leq 0$, the absolute value of eigenvalues of P are independent of $\alpha$ and 
\[|\mu| = \frac{1}{2}\left(\sqrt{(\lambda_{i}(A)\alpha - 1 - \eta)^2 + |(\lambda_{i}(A)\alpha - 1 - \eta)^2 - 4\eta|} \right)= \sqrt{\eta}\]
To ensure that $(\lambda_{i}(A)\alpha - 1 - \eta)^2 - 4\eta \leq 0$, we must have
\[(\alpha\lambda_{i}(A) + 1) - 2\sqrt{\lambda_{i}(A)\alpha} \leq \eta \leq (\alpha\lambda_{i}(A) +1) + 2\sqrt{\lambda_{i}(A)\alpha}\]
\[\left(1-\sqrt{\lambda_{i}(A)\alpha}\right)^2 \leq \eta \leq \left(1+\sqrt{\lambda_{i}(A)\alpha}\right)^2\]
For the spectral radius of $P$ to be $\sqrt{\eta}$, the above must hold for all $i$. We choose $\alpha$ as: 
\[\alpha \leq \left(\frac{2}{\sqrt{\lambda_{min}(A)} + \sqrt{\lambda_{max}(A)}}\right)^{2}\] 
and $\eta$ as:
\[(1-\sqrt{\lambda_{\min}(A)\alpha})^2\leq \eta \leq(1+\sqrt{\lambda_{\min}(A)\alpha})^2\]
Observe that if we choose the momentum parameter $\eta = \left(1-\sqrt{\lambda_{\min}(A)\alpha}\right)^2$, then $P$ has two repeated roots since $\sqrt{(\lambda_{i}(A)\alpha - 1 - \eta)^2 - 4\eta} = 0$. To ensure that $P$ does not have any repeated root we choose the momentum parameter $\eta = \left(1-\frac{\sqrt{\lambda_{\min}(A)\alpha}}{2}\right)^2$. Therefore, $\rho(P) = \left(1-\frac{\sqrt{\lambda_{\min}(A)\alpha}}{2}\right)$.
Using \textbf{Assumption \ref{A1}} and Lemma \ref{norm_upper_bound} we get
\[\mathbb{E}[\|\Tilde{X}_{n}\|^2] \leq \hat{C}^2\left(1-\frac{\sqrt{\lambda_{\min}(A)\alpha}}{2}\right)^{2n} \Lambda + \alpha^2\hat{C}^2K\sum_{i=0}^{n-1}\left(1-\frac{\sqrt{\lambda_{\min}(A)\alpha}}{2}\right)^{2(n-1-i)}(1 + \mathbb{E}[\|\Tilde{X}_{i}\|^{2}])\]
However, unlike in Case-1, here the constant $\hat{C}$ is not independent of $\alpha$ and $\lambda_{min}(A)$. The following lemma specifies an upper bound on $\hat{C}$.

\begin{lemma}
\label{c_hat_lemma}
\(\hat{C} \leq C\frac{5}{\sqrt{\alpha\lambda_{\min}(A)}}\), where $C$ is as defined in \eqref{C_thm1}.
\end{lemma}
\begin{proof}
    See Appendix \ref{C2}
\end{proof}
We define the sequence $\{V_{n}\}$ as follows
\[V_{n} = \hat{C}^2\left(1-\frac{\sqrt{\lambda_{\min}(A)\alpha}}{2}\right)^{2n} \Lambda + \alpha^2\hat{C}^2K\sum_{i=0}^{n-1}\left(1-\frac{\sqrt{\lambda_{\min}(A)\alpha}}{2}\right)^{2(n-1-i)}(1 + V_{i})\]
Observe that $\mathbb{E}[\|\Tilde{X}_{n}\|^2] \leq V_{n}$, and that $\{V_{k}\}$ satisfies 
\begin{gather*}
    V_{k+1} = \left(1-\frac{\sqrt{\lambda_{\min}(A)\alpha}}{2}\right)^{2}V_{k} + \hat{C}^{2}K\alpha^{2}(1 + V_{k}); \quad
    V_{0} = \hat{C}^{2}\Lambda
\end{gather*}
Therefore, we have
\[V_{k+1} = \left(\left(1-\frac{\sqrt{\lambda_{\min}(A)\alpha}}{2}\right)^{2} + \hat{C}^{2}K\alpha^{2}\right)V_{k} + \hat{C}^{2}K\alpha^{2}\]
To ensure that $\left(\left(1-\frac{\sqrt{\lambda_{\min}(A)\alpha}}{2}\right)^{2} + \hat{C}^{2}K\alpha^{2}\right) \leq \left(1-\frac{\sqrt{\lambda_{\min}(A)\alpha}}{4}\right)^{2}$ we need to choose $\alpha$ such that
\[1 + \frac{\lambda_{\min}(A)\alpha}{4} - \sqrt{\lambda_{\min}(A)\alpha} + \hat{C}^{2}K\alpha^{2} \leq 1 + \frac{\lambda_{\min}(A)\alpha}{16} - \frac{\sqrt{\lambda_{\min}(A)\alpha}}{2} \]
\[\mbox{or \quad } \frac{3\sqrt{\alpha}\lambda_{\min}(A)}{16} + \hat{C}^2K\alpha^{\frac{3}{2}} \leq \frac{\sqrt{\lambda_{\min}(A)}}{2}\]
Next, using Lemma \ref{c_hat_lemma}, the above can be ensured by choosing $\alpha$ such that
\[\frac{3\sqrt{\alpha}\lambda_{\min}(A)}{16} + C^2\frac{25}{\alpha\lambda_{\min}}K\alpha^{\frac{3}{2}} \leq \frac{\sqrt{\lambda_{\min}(A)}}{2}\]
\[\mbox{ or } \alpha \leq \left(\frac{\sqrt{\lambda_{\min}(A)}}{\frac{3}{8}\lambda_{\min}(A) + \frac{25CK}{\lambda_{\min}(A)}}\right)^{2}\]

The recursion for the sequence $\{V_{k+1}\}$ then follows
\begin{equation*}
    \begin{split}
        V_{k+1} & \leq \left(1-\frac{\sqrt{\lambda_{\min}(A)\alpha}}{4}\right)^{2}V_{k} + \hat{C}^{2}K\alpha^{2}
    \end{split}
\end{equation*}
Unrolling the recursion, we get
\begin{equation*}
    \begin{split}
         V_{n} & \leq \left(1-\frac{\sqrt{\lambda_{\min}(A)\alpha}}{4}\right)^{2n}V_{0} + \hat{C}^{2}K\alpha^{2} \sum_{i=0}^{n-1} \left(1-\frac{\sqrt{\lambda_{\min}(A)\alpha}}{4}\right)^{2i}\\
         & \leq \left(1-\frac{\sqrt{\lambda_{\min}(A)\alpha}}{4}\right)^{2n}V_{0} + \hat{C}^{2}K\alpha^{2} \frac{1}{1-\left(1-\frac{\sqrt{\lambda_{min}(A)}}{4}\right)^{2}}\\
         & \leq \left(1-\frac{\sqrt{\lambda_{\min}(A)\alpha}}{4}\right)^{2n}V_{0} + \hat{C}^{2}K\alpha^{2} \frac{4}{\sqrt{\alpha\lambda_{min}(A)}}
    \end{split}
\end{equation*}

Further it follows from $\alpha \leq \left(\frac{2}{\sqrt{\lambda_{min}(A)}+\sqrt{\lambda_{max}(A)}}\right)^2$ that $\alpha \leq \frac{1}{\lambda_{\min}(A)}$ and  $\left(1-\frac{\sqrt{\lambda_{\min}(A)\alpha}}{4}\right)^{2} \leq e^{-\frac{\sqrt{\lambda_{\min}(A)\alpha}}{2}}$.
\[V_{n} \leq  e^{-n\frac{\sqrt{\lambda_{\min}(A)\alpha}}{2}}\hat{C}^{2}\Lambda + \frac{4\alpha^{2}\hat{C}^{2}K}{\sqrt{\alpha\lambda_{min}(A)}}
\]
Again using Lemma \ref{c_hat_lemma},
\[V_{n} \leq  e^{-n\frac{\sqrt{\lambda_{\min}(A)\alpha}}{2}}\frac{25 C^{2}}{\lambda_{\min}(A)\alpha}\Lambda + \alpha^{2}\frac{100 C^{2}K}{\left(\lambda_{\min}(A)\alpha\right)^{3/2}}
\]
Observe that
\[\frac{e^{-n\frac{\sqrt{\lambda_{\min}(A)\alpha}}{2}}}{\lambda_{\min}(A)\alpha} \leq e^{-n\frac{\sqrt{\lambda_{\min}(A)\alpha}}{4}}\]
\[\mbox{ for } n \geq \frac{4}{\sqrt{\alpha\lambda_{\min}(A)}}\log\left(\frac{1}{\lambda_{\min}(A)\alpha}\right).\]
Let $n$ be as above. Then,
\[V_{n} \leq 25 C^{2}\Lambda e^{-\frac{n}{4}\sqrt{\lambda_{\min}(A)}\alpha} + \sqrt{\alpha}\frac{100 C^{2}K}{\left(\lambda_{\min}(A)\right)^{3/2}}\]

Choose $\alpha$ as below:
\[\alpha \leq \left(\frac{\epsilon(\lambda_{\min}(A))^{3/2}}{200 C^2 K}\right)^2\]
Then,
\[\sqrt{\alpha}\frac{100 C^{2}K}{\left(\lambda_{\min}(A)\right)^{3/2}} \leq \frac{\epsilon}{2} \Rightarrow \mathbb{E}[\|\Tilde{x}_{n}\|^2] \leq \mathbb{E}[\|\Tilde{X}_{n}\|^2] \leq V_{n} \leq \frac{\epsilon}{2} + \frac{\epsilon}{2} = \epsilon,\]
when $n$ is as follows:
\[n = \frac{4}{\sqrt{\alpha\lambda_{\min}(A)}}\log\left(\frac{50 C^2 \Lambda}{\epsilon}\right)\]

\[n = \max\left(\frac{4}{\sqrt{\alpha\lambda_{\min}(A)}}\log\left(\frac{50 C^2 \Lambda}{\epsilon}\right),\frac{4}{\sqrt{\alpha\lambda_{\min}(A)}}\log\left(\frac{1}{\lambda_{\min}(A)\alpha}\right)\right)\]

\textbf{Case-3: $\beta = 1$ (ASG)}

The proof progresses in a similar way as in Case-2. It is easy to see that the following relation holds with a modified definition of the matrix $P$.
\[
\mathbb{E}[\|\Tilde{X}_{n}\|^2] \leq \|P^n\|^2 \|\Tilde{X}_{0}\|^2 + \alpha^2K\sum_{i=0}^{n-1}\|P^{n-1-i}\|^2(1 + \mathbb{E}[\|\Tilde{X}_{i}\|^{2}]),
\]
where, 
\[ P \triangleq
\begin{bmatrix}
    I - \alpha A + \eta (I-\alpha A) & -\eta (I-\alpha A)   \\
    I & 0      
\end{bmatrix} 
\]
As in the previous case, we compute the eigen-values of $P$. The characteristic equation of $P$ is given by:
\[
det\left(
\begin{bmatrix}
    I - \alpha A + \eta (I-\alpha A) -\mu I & -\eta (I-\alpha A)  \\
    I & -\mu I      
\end{bmatrix} \right) = 0
\]
As in the previous case, this can be simplified to the following equation:
\begin{align*}
    det(- \mu I + \alpha\mu A - \mu\eta(I-\alpha A) + \mu^2I + \eta (I-\alpha A))  = 0
\end{align*}
We now further simplify the characteristic equation of $P$ to a more convienient form:
$$det\left(A - I \left(\frac{\mu+\eta\mu-\mu^2-\eta}{\alpha\mu+\alpha\mu\eta-\eta\alpha}\right)\right)=0$$
Progressing as in the previous case, we get that the eigen-values of $P$ satisfies:
\begin{gather*}
    \mu = \frac{-(\lambda_{i}(A)\alpha(1+\eta) - 1 - \eta) \pm \sqrt{(\lambda_{i}(A)\alpha(1+\eta) - 1 - \eta)^2 - 4\eta(1-\alpha\lambda_{i}(A))}}{2}
\end{gather*}
When $(\lambda_{i}(A)\alpha(1+\eta) - 1 - \eta)^2 - 4\eta(1-\alpha\lambda_{i}(A)) \leq 0$, we have that
\[|\mu| = \frac{1}{2}\left(\sqrt{(\lambda_{i}(A)\alpha(1+\eta) - 1 - \eta)^2 - (\lambda_{i}(A)\alpha(1+\eta) - 1 - \eta)^2 + 4\eta (1-\alpha\lambda_{i}(A))} \right)= \sqrt{\eta(1-\alpha\lambda_{i}(A))}\]
This implies that,
\begin{equation}
    \label{rho(P)ASG}
    \rho(P) = \sqrt{\eta(1-\alpha\lambda_{min}(A))}
\end{equation}
To ensure that $(\lambda_{i}(A)\alpha(1+\eta) - 1 - \eta)^2 - 4\eta(1-\alpha\lambda_{i}(A)) \leq 0$, we must have
\[\eta^2(1-\alpha\lambda_{i}(A))^2 + 2\eta(1-\alpha^2\lambda_{i}(A)^2) + (1-\alpha\lambda_{i}(A))^2 \leq 0\]
We assume $\alpha\leq \frac{1}{\lambda_{max}(A)}$ and therefore, $(1 - \alpha\lambda_{i}(A)) \geq 0$ holds for all $i$. Using this, the above relation simplifies to:
\[\eta^2(1-\alpha\lambda_{i}(A)) + 2\eta(1+\alpha\lambda_{i}(A)) + (1-\alpha\lambda_{i}(A)) \leq 0\]
For the above to hold, we must have that
\[\frac{2(1+\alpha\lambda_{i}(A)) - 4\sqrt{\alpha\lambda_{i}(A))}}{2(1-\alpha\lambda_{i}(A))}\leq \eta \leq \frac{2(1+\alpha\lambda_{i}(A)) + 4\sqrt{\alpha\lambda_{i}(A))}}{2(1-\alpha\lambda_{i}(A))}\]
\[\frac{(1-\sqrt{\alpha\lambda_{i}(A)})^2}{(1-\alpha\lambda_{i}(A))}\leq \eta \leq \frac{(1+\sqrt{\alpha\lambda_{i}(A)})^2}{(1-\alpha\lambda_{i}(A))}\]
The above must hold $\forall i$ and therefore we choose $\eta$ as:
\[\frac{(1-\sqrt{\alpha\lambda_{min}(A)})^2}{(1-\alpha\lambda_{min}(A))}\leq \eta \leq \frac{(1+\sqrt{\alpha\lambda_{min}(A)})^2}{(1-\alpha\lambda_{min}(A))}\]
As in Case-2, if we choose the momentum parameter $\displaystyle\eta = \frac{(1-\sqrt{\lambda_{\min}(A)\alpha})^2}{(1-\alpha\lambda_{min}(A))}$, then $P$ has two repeated roots. To ensure that $P$ does not have any repeated root we choose the momentum parameter $$\displaystyle\eta = \frac{\left(1-\frac{\sqrt{\lambda_{\min}(A)\alpha}}{2}\right)^2}{(1-\alpha\lambda_{min}(A))}.$$
Using \eqref{rho(P)ASG}, we get
$\rho(P) = \left(1-\frac{\sqrt{\lambda_{\min}(A)\alpha}}{2}\right)$ which is same as in Case-2 and therefore we have:
\[\mathbb{E}[\|\Tilde{X}_{n}\|^2] \leq \tilde{C}^2\left(1-\frac{\sqrt{\lambda_{\min}(A)\alpha}}{2}\right)^{2n} \Lambda + \alpha^2\Tilde{C}^2K\sum_{i=0}^{n-1}\left(1-\frac{\sqrt{\lambda_{\min}(A)\alpha}}{2}\right)^{2(n-1-i)}(1 + \mathbb{E}[\|\Tilde{X}_{i}\|^{2}]).\]
The above expression is same as that in Case-2 except the term $\Tilde{C}$. Since, the matrix $P$ is different when $\beta = 1$, $\hat{C}$ in Case-2 need not be equal to $\Tilde{C}$ in Case-3.
However, we next show that $\Tilde{C}$ follows the exact same upper bound as in Case-2, Lemma~\ref{c_hat_lemma}. Towards this, we have the following lemma:
\begin{lemma}
\label{c_tilde_lemma}
\(\Tilde{C} \leq C\frac{5}{\sqrt{\alpha\lambda_{\min}(A)}}\), where $C$ is as defined in \eqref{C_thm1}.
\end{lemma}
\begin{proof}
    See Appendix \ref{C3}
\end{proof}
Thereafter, we can proceed exactly as in case-2 to prove the theorem.
\section{Proof of Auxilary Lemmas}
\subsection{Proof of Lemma \ref{norm_upper_bound}}
\label{Lemma5_proof}
As in \citet{foucart}, we first construct a matrix norm $\vertiii{\cdot}$ such that $\vertiii{M} = \rho(M) + \delta$. Consider the Jordan canonical form of $M$
\begin{center}
    \[M = S\begin{bmatrix}
        J_{n_{1}}(\lambda_{1}(M)) & 0 &\ldots & 0\\
        0 & J_{n_{2}}(\lambda_{2}(M)) & \ddots & \vdots\\
        \vdots & \ddots & \ddots & 0 \\
        0 & \ldots & 0 & J_{n_{k}}(\lambda_{k}(M))
    \end{bmatrix} S^{-1}\]
\end{center}
We define 
\begin{center}
    \[D(\delta) =
    \begin{bmatrix}
        D_{n_{1}}(\delta) & 0 &\ldots & 0\\
        0 & D_{n_{2}}(\delta) & \ddots & \vdots\\
        \vdots & \ddots & \ddots & 0 \\
        0 & \ldots & 0 & D_{n_{k}}(\delta)
    \end{bmatrix},\] where \[D_{j}(\delta) = 
    \begin{bmatrix}
        \delta & 0 & \ldots & 0\\
        0 & \delta^{2} & \ddots & \vdots\\
        \vdots & \ddots & \ddots & 0 \\
        0 & \ldots & 0 & \delta^{j}
    \end{bmatrix}\]
\end{center}
Therefore,
\begin{center}
    \[D(\frac{1}{\delta})S^{-1}MSD(\delta) = 
    \begin{bmatrix}
        B_{n_{1}}(\lambda_{1}(M),\delta) & 0 &\ldots & 0\\
        0 & B_{n_{2}}(\lambda_{2}(M),\delta) & \ddots & \vdots\\
        \vdots & \ddots & \ddots & 0 \\
        0 & \ldots & 0 & B_{n_{k}}(\lambda_{k}(M),\delta)
    \end{bmatrix}\]
\end{center}
where,
\begin{center}
    \[B_{i}(\lambda,\delta) = D_{i}(\frac{1}{\delta})J_{i}(\lambda)D_{i}(\delta) = 
    \begin{bmatrix}
        \lambda & \delta & 0 &\ldots & 0\\
        0 & \lambda & \delta & \ddots & \vdots\\
        0 & \ddots & \ddots & \ddots & 0 \\
        \vdots & \ddots & \ddots & \lambda & \delta \\
        0 & \ldots & 0 & 0 & \lambda
    \end{bmatrix}\]
\end{center}
We define the matrix norm $\vertiii{\cdot}$ as
\[\vertiii{M} \triangleq \|D(\frac{1}{\delta})S^{-1}MSD(\delta)\|_{1}\]
where $\|\cdot\|_{1}$ is the matrix norm induced by the vector $L_{1}$-norm. Using the fact that $\|M\|_{1} = \max_{j\in[1:d]}\sum_{i=1}^{d}|m_{i,j}|$, where $m_{i,j}$ is the $i,j$-th entry of $M$, we have
\[\vertiii{M} = \max_{j\in [1:d]}(|\lambda_{j}| + \delta) = \rho(M) + \delta.\]
\[\mbox{ and }\vertiii{M^{n}} \leq \vertiii{M}^{n} \leq (\rho(M)+\delta)^{n}\]
It can be easily seen that $\|M\|_{1} \geq \frac{1}{\sqrt{d}}\|M\|_{2}$. Therefore it follows that
\begin{equation*}
    \begin{split}
        \vertiii{M} & = \|D\left(\frac{1}{\delta}\right)S^{-1}MSD(\delta)\|_{1}\\
        & \geq \frac{1}{\sqrt{d}} \|D\left(\frac{1}{\delta}\right)S^{-1}MSD(\delta)\|_{2}\\
        &\geq \frac{1}{\sqrt{d}} \sigma_{\min}\left(D\left(\frac{1}{\delta}\right)\right)\sigma_{\min}(S^{-1})\|M\|_{2}\sigma_{\min}(S)\sigma_{\min}(D(\delta))
    \end{split}
\end{equation*}
where $\sigma_{\min}(\cdot)$ denotes the smallest singular value and we have used the fact that $\|AB\| \geq \sigma_{\min}(A)\|B\|$ repeatedly. For $\delta < 1$, and $r$ defined as the size of largest Jordan block of $M$ \[\sigma_{\min}\left(D\left(\frac{1}{\delta}\right)\right)\sigma_{\min}\left(D\left(\delta\right)\right) = \delta^{r}\frac{1}{\delta} = \delta^{r-1}.\]
We conclude the first half of the lemma by defining $C_{\delta}$ as $\frac{\sqrt{d}}{\delta^{r-1}\sigma_{\min}(S)\sigma_{\min}(S^{-1})}$.

In the case that the eigenvalues are distinct, $r=1$ and $C_\delta$ defined above becomes independent of $\delta.$ Moreover, when all eigen-values are distinct, each Jordan block is $J_{n_i}(\lambda_{i}(M)) = [\lambda_{i}(M)]$ and the second half follows.

\subsection{Proof of Lemma \ref{c_hat_lemma}}
\label{C2}
Let $S$ be the matrix in Jordan decomposition of $A$, i.e., $SAS^{-1} = D$, where $D$ is a diagonal matrix with eigenvalues of $A$ as its diagonal elements. Then,

\[
\begin{pmatrix}
    S & 0 \\
    0 & S
\end{pmatrix} P
\begin{pmatrix}
    S^{-1} & 0 \\
    0 & S^{-1}
\end{pmatrix} = 
\begin{pmatrix}
    I - \alpha S A S^{-1}+\eta I & -\eta SS^{-1}\\
    SS^{-1} & 0
\end{pmatrix} = 
\begin{pmatrix}
    I - \alpha D +\eta I & -\eta I\\
    I & 0
\end{pmatrix}, 
\]
where $0_d$ is the zero matrix of dimension $d \times d$. For ease of exposition, suppose $A$ is a $2 \times 2$ matrix with eigenvalues $\lambda_1$ and $\lambda_{2}$. Then 
\[
\begin{pmatrix}
    S & 0 \\
    0 & S
\end{pmatrix} P
\begin{pmatrix}
    S^{-1} & 0 \\
    0 & S^{-1}
\end{pmatrix} = 
\begin{pmatrix}
    1 + \eta - \alpha \lambda_{1} & 0 & -\eta & 0\\
    0 & 1 + \eta - \alpha \lambda_{2} & 0 & -\eta\\
    1 & 0 & 0 & 0\\
    0 & 1 & 0 & 0
\end{pmatrix}
\]
Suppose $E$ is the elementary matrix associated with the exchange of row-2 and row-3. 
It is easy to see that $E=E^T=E^{-1}$ and that,
\[
E
\begin{pmatrix}
    S & 0 \\
    0 & S
\end{pmatrix} P
\begin{pmatrix}
    S^{-1} & 0 \\
    0 & S^{-1}
\end{pmatrix}E^{-1} = 
\begin{pmatrix}
    1 + \eta - \alpha \lambda_{1} & -\eta & 0 & 0\\
    1 & 0 & 0 & 0\\
    0 & 0 & 1 + \eta - \alpha \lambda_{2} & -\eta\\
    0 & 0 & 1 & 0
\end{pmatrix}= 
\begin{pmatrix}
    B_{1} & 0 \\
    0 & B_{2}
\end{pmatrix}
\]
where, 
\[B_{i} = 
\begin{pmatrix}
    1 + \eta - \alpha \lambda_{i} & -\eta\\
    1 & 0
\end{pmatrix}
\]
Suppose, \(X_{i} = \begin{pmatrix}
    x_{i,1} & x_{i,2}\\
    x_{i,3} & x_{i,4}
\end{pmatrix}\) and,
\[
X_{i}^{-1}B_{i}X_{i} = 
\begin{pmatrix}
    \mu_{i,+} & 0\\
    0 & \mu_{i,-}
\end{pmatrix}.
\]
Here $\mu_{i,+} = \frac{(1 - \alpha\lambda_{i} + \eta) + \sqrt{\Delta_{i}}}{2}$ and $\mu_{i,-} =  \frac{(1 - \alpha\lambda_{i} + \eta) - \sqrt{\Delta_{i}}}{2}$, where $\Delta_i = (1+\eta - \alpha\lambda_{i})^2 - 4\eta$. Solving the above equation we get, 
\[
X_{i} = \begin{pmatrix}
    x_{i,3}\mu_{i,+} & x_{i,4}\mu_{i,-}\\
    x_{i,3} & x_{i,4}
\end{pmatrix}
\]
Setting $x_{i,3} = x_{i,4} = 1$,
\[
X_{i} = \begin{pmatrix}
    \mu_{i,+} & \mu_{i,-}\\
    1 & 1
\end{pmatrix} \mbox{ and }
X_{i}^{-1} = 
\frac{1}{\mu_{i,+} - \mu_{i,-}}
\begin{pmatrix}
    1 & -1\\
    -\mu_{i,-} & \mu_{i,+}\\
\end{pmatrix} 
\]
For a general $d\times d$ matrix $A$, using a similar procedure, it can be shown that 
\[
\begin{pmatrix}
    X_{1} & 0 & 0\\
    0 & \ddots & \vdots\\
    0 & \cdots &X_{d}\\
\end{pmatrix}
E_{d\times d}
\begin{pmatrix}
    S & 0 \\
    0 & S
\end{pmatrix} P
\begin{pmatrix}
    S^{-1} & 0 \\
    0 & S^{-1}
\end{pmatrix}E_{d\times d}^{-1} 
\begin{pmatrix}
    X_{1}^{-1} & 0 & 0\\
    0 & \ddots & \vdots\\
    0 & \cdots &X_{d}^{-1}\\
\end{pmatrix} 
\]
\[= 
\begin{pmatrix}
    \mu_{1,+} &0 &0 &\cdots& 0\\
    0 &\mu_{1,-}& 0& \cdots& 0\\
    \vdots& 0 &\ddots& \cdots& 0\\
    0& \cdots& 0& \mu_{d,+}& 0\\
    0& \cdots& 0& 0&\mu_{d,-}
\end{pmatrix}
\]
where $E_{d\times d}$ and $E_{d\times d}^{-1}$ are permutation matrices that transform the matrix between them to a block diagonal matrix.
Let \(\hat{S} = \begin{pmatrix}
    X_{1} & 0 & 0\\
    0 & \ddots & \vdots\\
    0 & \cdots &X_{d}\\
\end{pmatrix}
E_{d\times d}
\begin{pmatrix}
    S & 0 \\
    0 & S
\end{pmatrix}\).
Therefore, 
\[
\hat{C} = \frac{\sqrt{d}}{\sigma_{\min}(\hat{S}) \sigma_{\min}(\hat{S}^{-1})}
\]
In order to simplify the expression for $\hat C$, we require the following lemma:
\begin{lemma}
\label{singular-val}
For all invertible matrices $M$ of order $d \times d$, the following identity holds:
 \[\frac{1}{\sigma_{d}(M)\sigma_{d}(M^{-1})} = \sigma_{1}(M)\sigma_{1}(M^{-1}),\]
 where $\sigma_1(X) \geq \cdots \geq \sigma_d(X)$ denote the singular values of $X$.
\end{lemma}

\begin{proof}
By definition, $\sigma_1^2(M) \geq \cdots \geq \sigma_d^2(M)$ are the eigenvalues of $M^{T}M$. Then, the eigenvalues of $(M^TM)^{-1}=M^{-1}(M^{-1})^T$ are $\frac{1}{\sigma_d(M)^2} \geq \cdots \geq \frac{1}{\sigma_1(M)^2}.$ Note that $M^{-1}(M^{-1})^T$ and $(M^{-1})^TM^{-1}$ are similar since $(M^{-1})^TM^{-1}=M(M^{-1}(M^{-1})^T)M^{-1}$. Consequently, 
$M^{-1}(M^{-1})^{T}$ and $(M^{-1})^{T}M^{-1}$ have the same set of eigenvalues and we find that the singular values of $M^{-1}$ are $\frac{1}{\sigma_d(M)} \geq \cdots \geq \frac{1}{\sigma_1(M)}.$

Thus,
$$\sigma_1(M^{-1})=\frac{1}{\sigma_d(M)},$$
$$\sigma_d(M^{-1})=\frac{1}{\sigma_1(M)}$$
and the result follows.
\end{proof}

Using Lemma \ref{singular-val},
we have 
\begin{equation*}
    \begin{split}
        \hat{C} &= \sqrt{d}\sigma_{\max}(\hat{S}) \sigma_{\max}(\hat{S}^{-1})\\
        &\leq \sqrt{d} \max_{i}\{\sigma_{\max}(X_{i})\}\sigma_{\max}(S)\sigma_{\max}(S^{-1})\max_{i}\{\sigma_{\max}(X_{i}^{-1})\}\\
        & = C \max_{i}\{\sigma_{\max}(X_{i})\}\max_{i}\{\sigma_{\max}(X_{i}^{-1})\},
    \end{split}
\end{equation*}
where $C$ is as defined in \eqref{C_thm1}. Now, for any matrix $X$ of order $d \times d$, $$\sigma_{\max}(X)=\|X\|_2\leq \|X\|_F=(\sum_{i,j}|x_{ij}|^2)^{1/2} \leq d\max_{i,j}|x_{ij}|,$$
where $\|\cdot\|_F$ denotes the Frobenius norm. Using the above inequality, $\sigma_{\max}(X_{i}) \leq 2$ and $\sigma_{\max}(X_{i}^{-1}) \leq \frac{2}{|\mu_{i,+}-\mu_{i,-}|}= \frac{2}{|\sqrt{\Delta_{i}}|}$. Next we lower bound $|\sqrt{\Delta_{i}}|$.

For a complex number $z$, observe that $|\sqrt{z}|$ = $\sqrt{|z|}$. Now,
\begin{align}
        \label{delta_i}
        |\Delta_{i}| &= 4\eta - (1+\eta-\alpha\lambda_{i}(A))^2\\
        &\geq  4\eta - (1+\eta-\alpha\lambda_{min}(A))^2 \nonumber
\end{align}
Using $\eta = \left(1 - \frac{\sqrt{\alpha\lambda_{\min}}}{2}\right)^{2}$,
\begin{equation*}
    \begin{split}
        |\Delta_{i}| & \geq 4\eta - \left(1+ 1 + \frac{\alpha\lambda_{\min}}{4} - \sqrt{\alpha\lambda_{\min}} -\alpha\lambda_{\min}\right)^{2}\\
        & = 4\eta - (2\sqrt{\eta} - \frac{3\alpha}{4}\lambda_{\min})^{2}\\
        & = 4 \left[ (\sqrt{\eta})^{2} - \left(\sqrt{\eta} - \frac{3\alpha\lambda_{\min}}{8}\right)^{2}  \right]\\
        & = \frac{3\alpha\lambda_{\min}}{2}\left(2\sqrt{\eta} - \frac{3\alpha\lambda_{\min}}{8}\right)\\
        & = \frac{3\alpha\lambda_{\min}}{2} \left(2 - \sqrt{\alpha\lambda_{\min}} - \frac{3\alpha\lambda_{\min}}{8}\right)\\
        & \geq \frac{3\alpha\lambda_{\min}}{2}\left(2 - 1 - \frac{3}{8}\right) \\
        & = \frac{15}{16}\alpha\lambda_{\min}
    \end{split}
\end{equation*}
Using this, we get $\max_{i}\{\sigma_{\max}(X_{i}^{-1})\} \leq 2\sqrt{\frac{16}{15}}\frac{1}{\sqrt{\alpha\lambda_{\min}}}$, and therefore
\(\hat{C} \leq \frac{5C}{\sqrt{\alpha\lambda_{\min}}}\)

\subsection{Proof of Lemma \ref{c_tilde_lemma}}
\label{C3}
The proof of the lemma can proceed exactly as in the proof of Lemma~\ref{C2}. The only difference is that one needs to lower bound $|\Delta_i| = 4\eta(1-\alpha\lambda_{i}(A)) - (\alpha(1+\eta)\lambda_{i}(A)-1-\eta)^2$ for $\eta = \frac{\left(1 - \frac{\sqrt{\alpha\lambda_{\min}}}{2}\right)^{2}}{(1-\alpha\lambda_{i}(A))}.$ We define $\eta' = \left(1 - \frac{\sqrt{\alpha\lambda_{\min}}}{2}\right)^{2}$ and therefore $\eta = \frac{\eta'}{(1-\alpha\lambda_{i}(A))}$. Using this, we get
\begin{align*}
    |\Delta_{i}| &= 4\eta(1-\alpha\lambda_{i}(A)) - (1+\eta-\alpha(1+\eta)\lambda_{i}(A))^2\\
    &= 4\eta' - (1+\eta-\alpha\lambda_{i}(A)(1+\eta))^2\\
    &= 4\eta' - \Big((1+\eta)(1-\alpha\lambda_{i}(A))\Big)^2\\
    &= 4\eta' - \bigg(\Big(1+\frac{\eta'}{1-\alpha\lambda_{i}(A)}\Big)(1-\alpha\lambda_{i}(A))\bigg)^2\\
    &= 4\eta' - (1+\eta'-\alpha\lambda_{i}(A))
\end{align*}
The expression for $|\Delta_{i}|$ is exactly same as in the proof of Lemma~\ref{C2} (cf. \eqref{delta_i}). Thereafter, we proceed as in proof of Lemma~\ref{C2} to prove the claim.

\end{document}